\documentclass[10pt,journal]{IEEEtran}

\ifCLASSOPTIONcompsoc

\usepackage[nocompress]{cite}
\else

\usepackage{cite}
\fi

\ifCLASSINFOpdf

\else

\fi

\usepackage{subfigure}
\usepackage{graphicx}
\usepackage{balance}
\usepackage{amsmath,amsthm, amssymb}
\usepackage{amssymb,wasysym}
\usepackage{amsfonts,balance}
\usepackage{mathrsfs}
\usepackage{cite}
\usepackage{color}
\usepackage{xcolor}
\usepackage{url}
\usepackage{bm}
\usepackage{wrapfig}
\usepackage{verbatim}
\usepackage{dsfont}
\usepackage{amsmath}
\usepackage{algorithm}
\usepackage{algpseudocode}
\usepackage{xurl}
\usepackage{tabularx}
\usepackage{enumitem}
\usepackage[colorlinks=true, linkcolor=black, citecolor=black, urlcolor=black, bookmarks=false]{hyperref}

\theoremstyle{definition}
\newtheorem{theorem}{\textbf{Theorem}}

\newtheorem{proposition}{\textbf{Proposition}}
\newtheorem{lemma}{\textit{Lemma}}

\newtheorem{definition}{\textbf{Definition}}

\theoremstyle{definition}
\newtheorem{remark}{\textbf{Remark}}
\allowdisplaybreaks[3]

\ifodd 1
\else

\fi

\graphicspath{{figures/}}

\begin{document}

	\title{SlimCaching: Edge Caching of Mixture-of-Experts for Distributed Inference}
	
	\author{Qian Chen, \IEEEmembership{Member, IEEE}, Xianhao Chen, \IEEEmembership{Member, IEEE}, and Kaibin Huang, \IEEEmembership{Fellow,~IEEE}
    \thanks{
    Q. Chen, X. Chen, and K. Huang are with the Department of Electrical and
Computer Engineering, The University of Hong Kong, Hong Kong (Email: qchen@eee.hku.hk, xcheneee@hku.hk, and huangkb@hku.hk). Corresponding authors: X. Chen and K. Huang.}
	}

	\markboth{Journal of \LaTeX\ Class Files}%
	{Shell \MakeLowercase{\textit{et al.}}: Bare Advanced Demo of IEEEtran.cls for IEEE Computer Society Journals}

	
	\IEEEtitleabstractindextext{
		\begin{abstract}

Mixture-of-Experts (MoE) models improve the scalability of large language models (LLMs) by activating only a small subset of relevant experts per input. However, the sheer number of expert networks in an MoE model introduces a significant storage/memory burden for an edge device. To address this challenge, we consider a scenario where experts are dispersed across an edge network for distributed inference. Based on the popular Top-$K$ expert selection strategy, we formulate a latency minimization problem by optimizing expert caching on edge servers under storage constraints. When $K=1$, the problem reduces to a monotone submodular maximization problem with knapsack constraints, for which we design a greedy-based algorithm with a $(1 - 1/e)$-approximation guarantee. For the general case where $K\geq1$, expert co-activation within the same MoE layer introduces non-submodularity, which renders greedy methods ineffective. To tackle this issue, we propose a successive greedy decomposition method to decompose the original problem into a series of subproblems, with each being solved by a dynamic programming approach. Furthermore, we design an accelerated algorithm based on the max-convolution technique to obtain the approximate solution with a provable guarantee in polynomial time. Simulation results on various MoE models demonstrate that our method significantly reduces inference latency compared to existing baselines.
		\end{abstract}
		
		\begin{IEEEkeywords}
			Edge AI, large language models, mixture-of-experts, expert caching, edge inference.
	\end{IEEEkeywords}}
	
	\maketitle

	\IEEEdisplaynontitleabstractindextext
	\IEEEpeerreviewmaketitle
	
	\section{Introduction}\label{sec:intro}
Large language models (LLMs), such as GPT~\cite{openai2022chatgpt} and LLaMA~\cite{touvron2023llama}, have delivered remarkable performance across various tasks, including many privacy-sensitive and real-time applications~\cite{10976336,10764961,10845862}. Driven by the growing demand for privacy preservation and low-latency responses, there is a pivotal trend toward deploying LLMs at the network edge~\cite{xu2024device,lin2024splitlora}. Reflecting this trend, leading technology companies such as Qualcomm, Huawei, and Apple have already integrated on-device LLMs into consumer-grade mobile devices~\cite{10835069,11152695}. This movement toward edge deployment is expected to shape the future of LLMs and fundamentally transform the landscape of AI-powered mobile applications.

Despite these advances, LLMs at the network edge face inherent performance limitations compared with the cloud-based counterparts~\cite{10.1145/3719664,lin2025hsplitlora}. This is because the capabilities of LLMs generally scale with model size and computing resources~\cite{kaplan2020scaling}, making it challenging to achieve high performance under edge resource constraints. To overcome the scaling challenges, the Mixture-of-Experts (MoE) architecture has emerged as a predominant solution~\cite{jiang2024mixtral}, which has been adopted in numerous state-of-the-art LLMs such as Phi-3.5-MoE, DeepSeek-V3, and Hunyuan-Large~\cite{10937907}. Specifically, MoE models replace each \textit{dense} transformer block with a massive set of expert networks (e.g., feedforward networks (FFNs)), among which only a subset is activated for each input token, as visualized in Fig. \ref{fig:visualization}. This paradigm allows LLMs to scale in size without a proportional increase in computation for each inference, thereby enabling an effective balance between model performance and computational efficiency -- an essential requirement for deployment at the network edge.

However, while MoE architectures substantially reduce inference workload, they pose significant \textit{storage/memory} challenges for edge devices. For instance, the Switch Transformer architecture allows each MoE layer to host up to hundreds of experts, leading to an MoE model up to 65 times larger than a FLOPs-equivalent dense T5 model, from 446 MB to 29.4 GB~\cite{fedus2022switch}. Such an explosion in parameters makes them ill-suited for edge devices with limited storage capacity~\cite{10630945,10746555}, such as the iPhone 16 with only 128 GB basic storage. In practice, only a small fraction of parameters can be cached locally on mobile devices, thereby constraining the scale of MoE models for on-device deployment.

To mitigate device-side burdens, split inference (SI) frameworks have emerged as a potential solution.
A representative design is the U-shaped SI framework, which stores the head and tail layers at the user and executes the body layers at the edge server or the cloud~\cite{10579544,10107635}. Although this framework keeps raw inputs and final outputs local to realize privacy benefits, each token incurs a fixed communication cost equivalent to twice the size of its hidden-state vector due to one upload to and one download from the server, which is a substantial overhead for long-context tasks. Moreover, since LLMs typically employ Transformer blocks that share the same architecture and hidden dimension, this per-token communication cost remains essentially constant under this U-shaped SI scheme, regardless of how the model layers are allocated between the user and server. When the user and edge server cannot store all model parameters, each token incurs an additional and excessive communication latency. Under such conditions, the resulting communication pattern is equivalent to that of cloud MoE inference, where all MoE layers are executed at the cloud and each hidden state must traverse both the user–edge link and the edge–cloud backhaul.

Considering that the majority of parameters in MoE models are contributed by expert networks\footnote{Considering the Switch Transformer model family, when the number of expert networks per MoE layer increases from 8 to 32, the proportion of the MoE model’s total data size occupied by the expert networks rises from 82.3\% to 94.4\%.}, a more effective strategy is to distribute experts, rather than layers, across users and edge networks. Empirical studies show that deeper MoE layers exhibit highly skewed activation patterns, where only a small subset of experts is consistently selected while most remain rarely activated during inference. This implies that each user can satisfy most computing needs by caching a small set of preferred experts locally.
To this end, we propose a framework called di\underline{s}tributed \underline{l}ow-latency \underline{i}nference in \underline{M}oE with expert caching (SlimCaching), where each device stores the user’s preferred experts alongside non-expert components, while the wireless edge network caches the remaining experts of various MoE models.
When the required experts of an MoE layer are locally available, the layer can be executed entirely on the user device without any communication.
When desired experts are not locally available, users route the hidden states of their tokens to nearby wireless edge servers with the corresponding experts for inference. This distributed inference framework for MoE models offers several salient advantages: \textbf{1) Privacy:} local data and final predictions remain on the user side with privacy benefits, akin to on-device inference, with only intermediate hidden states uploaded to edge servers; \textbf{2) Storage/memory efficiency:} each user only stores minimal, ``slim'' version of the MoE models, consisting of frequently activated experts and all non-expert components, significantly reducing local storage and memory costs; \textbf{3) Communication efficiency:} compared with U-shaped SI schemes, SlimCaching reduces communication traffic, as tokens can be processed entirely locally when all activated experts are available on edge devices.
As validated in Fig. \ref{fig:commcost_compare}, SlimCaching outperforms U-shaped SI in terms of communication efficiency, with the performance gap widening as user storage capacity increases.
By contrast, the communication latency of U-shaped SI remains constant, since each token incurs an excessive communication latency once cloud access is triggered, which is also the case for cloud MoE inference.

	\begin{figure}[t]
		\centering
		\includegraphics[width = 0.3\textwidth]{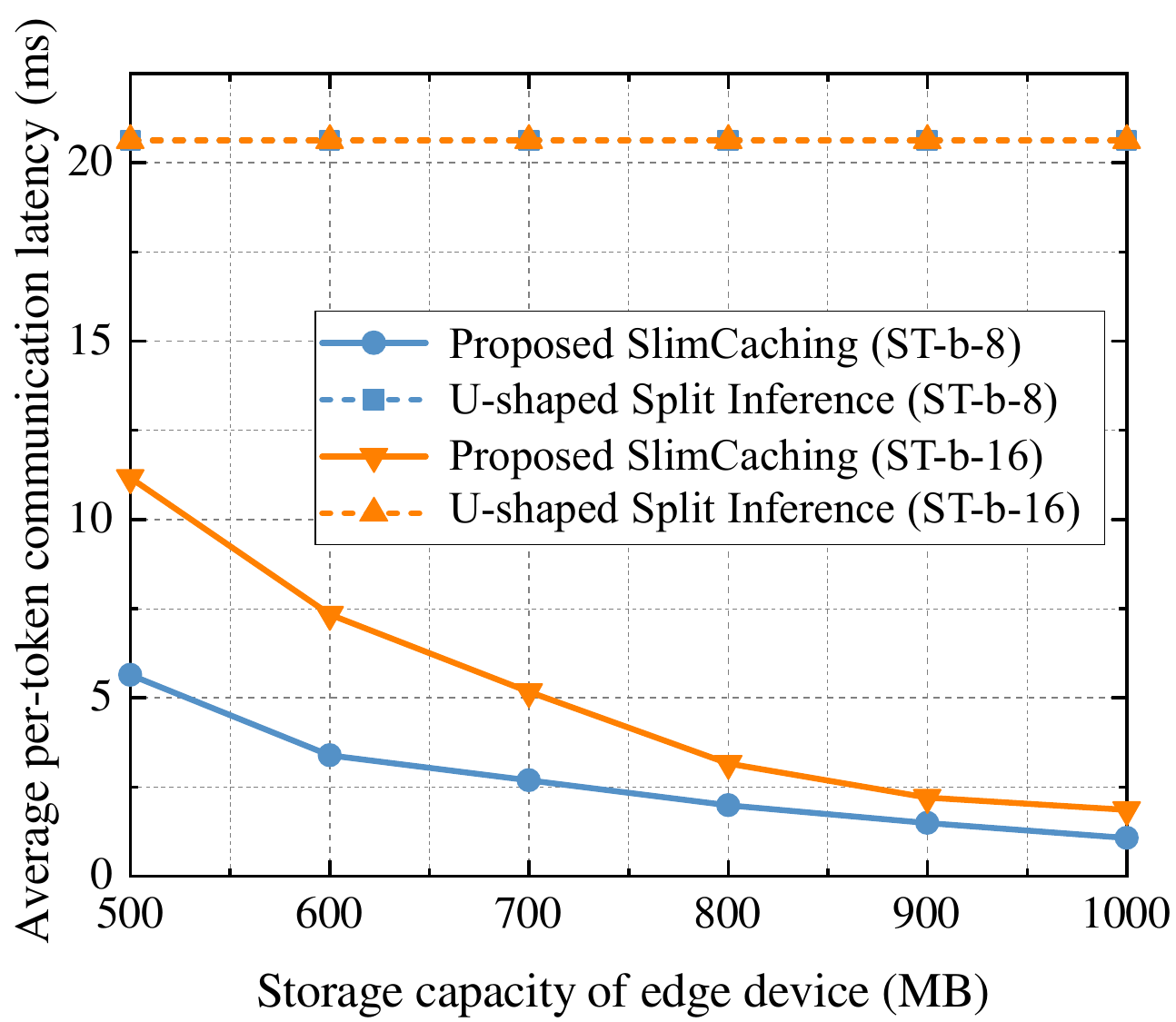}
		\caption{Comparison of the average per-token communication latency of the proposed SlimCaching framework and the U-shaped SI scheme across different device storage capacities in a scenario consisting of a single user, a single edge server, and the cloud. The expert activation statistics are computed from the prompts of the Visual Question Answering (VQA) v2 dataset, and “ST-b-X” denotes a Switch Transformer–based MoE model with X experts per MoE layer. The storage capacity of the edge server is set to 1.5 GB and other simulation parameters follow the settings described in Section \ref{sec:experiment}.\label{fig:commcost_compare}}
	\end{figure}

\begin{figure}[t]
    \centering
    \includegraphics[width = 0.4\textwidth]{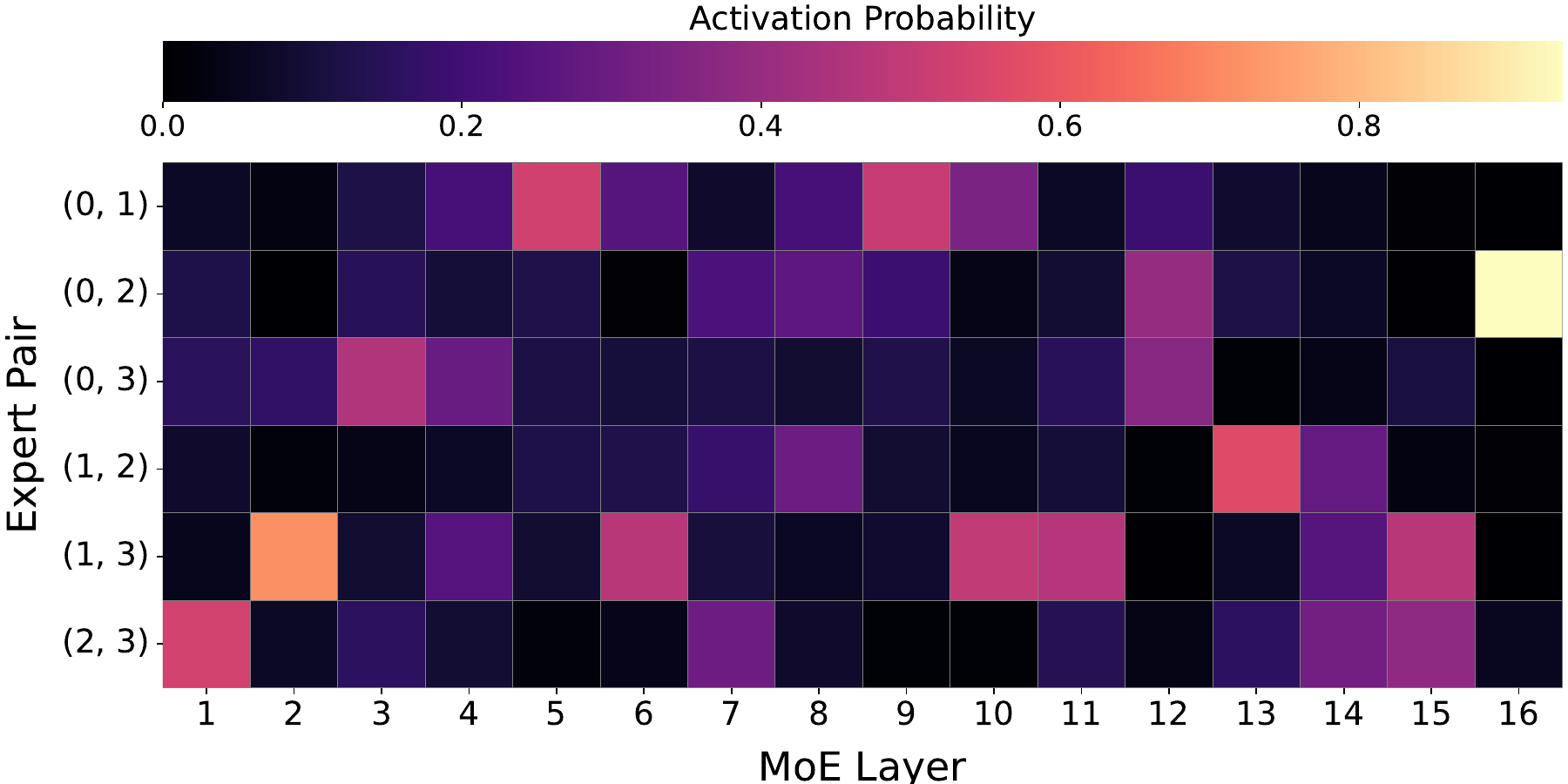}
    \caption{Visualization of activated experts in different MoE layers of the text part in MoE-LLaVA-Phi2-2.7B-4e with Top-2 strategy~\cite{11353922} under the Science Question Answering (SQA) dataset~\cite{lu2022learn}.}
    \label{fig:visualization}
\end{figure}

The proposed paradigm above raises a fundamental research question: \textit{Given the popularity of experts, how can they be optimally placed across distributed storage-constrained edge servers to enable latency-efficient MoE inference?} At the first glance, the problem falls into traditional edge caching problems, which determine the placement of content on edge servers based on their content popularity. Nonetheless, key differences exist. In classical content caching or dense model placement, cached items are independently retrievable~\cite{10433726,10402004,10219010}. In contrast, the activated experts within the same MoE layer exhibit a strong correlation since the model activates $K$ experts per input (typically based on the Top-$K$ strategy). In particular, when multiple selected experts at the same layer are colocated on a single server, they share a single uploaded input hidden state, but each activated expert generates a distinct output that must be returned. As a result, the overall layer latency is no longer a simple linear summation of individual expert latencies. The coupling effect among the selected experts introduces non-submodularity into the optimization problem, making it significantly more challenging. Consequently, existing delay-optimal content placement schemes, which are often based on submodular optimization and greedy procedures~\cite{6600983}, are no longer effective for MoE placement in terms of theoretical approximation guarantees when $K > 1$.

To address the challenging problem, in this paper, we design a SlimCaching framework under Top-$K$ strategy.
Specifically, given the limited storage capacity of edge servers, we aim to determine optimal expert placement strategies that minimize the average inference latency across all users. The main contributions of this work are summarized as follows:
\begin{itemize}
    \item  
    We define a novel expert caching problem tailored for distributed MoE inference. Considering a multi-edge system with cooperative caching between edge devices and edge servers, we formulate a combinatorial optimization problem to minimize the average inference latency subject to storage capacity constraints. 
    We identify that the resulting problem is a monotone submodular maximization with multiple knapsack constraints when $K = 1$, and becomes a monotone but non-submodular and non-supermodular maximization problem when $K > 1$ due to expert dependencies.

    \item 
For the case of $K = 1$, we develop a greedy-based algorithm that achieves a $(1 - 1/e)$-approximation guarantee. For the more general setting with $K \geq 1$, we propose a successive greedy decomposition approach that reformulates the problem into a sequence of subproblems, each solvable via a dynamic programming (DP)-based algorithm. The proposed algorithm attains a $(1 - \kappa_g)/2$-approximate global solution in polynomial time, where $\kappa_g$ is the supermodular curvature. Furthermore, we propose an accelerated algorithm considering the uniform expert sizes within each MoE model.

    \item We conduct extensive experiments on SQA and Visual Question Answering (VQA)-v2 datasets using diverse MoE models with varying Top-$K$ configurations. The results demonstrate that the proposed method consistently outperforms traditional greedy baselines in terms of both latency and computational efficiency.

\end{itemize}

The rest of this paper is organized as follows. We review the related works in Section \ref{sec:related_works}. Then, we provide the system models in Section \ref{sec:system_model}, including the expert activation procedure during MoE inference, expert caching model, and per-token inference latency. Based on the activation probability of experts, we formulate a latency-minimization optimization problem and characterize its structural property in Section \ref{sec:formulation}. In Section \ref{sec:proposed_alg}, we develop the algorithms for the special case $K=1$ and for the general case $K \geq 1$. Experimental results are presented in Section \ref{sec:experiment}, followed by concluding remarks in Section \ref{sec:conclusion}.

\section{Related Works}\label{sec:related_works}

The relevant research works on MoE and its edge deployment can be summarized as follows.

1) \textit{Expert selection}: 
For MoE-based models, expert selection plays a central role in determining both inference latency and accuracy. The most popular method is Top-$K$ expert selection, where, for each input, the $K$ experts with the highest routing probabilities are chosen to process the data~\cite{out_sparse_moe,gshard_iclr}. There are also dynamic expert selection strategies, where a threshold is defined such that experts are selected if the cumulative routing probabilities exceed the threshold~\cite{huang-etal-2024-harder}. In this work, we adopt the widely used Top-$K$ expert selection to study the expert placement problem in edge networks.

\textit{2) In-memory expert caching:} One of the main bottlenecks in MoE inference arises from the high memory overhead associated with storing a large number of expert networks. One usual practice is on-demand expert fetching, where inactive experts are kept in low-tier memory and only loaded into GPU memory when selected. However, since the router must first determine the Top-$K$ experts before initiating expert loading, the fetching process is fully serialized with computation, causing substantial latency. To address this issue, an expert prefetching mechanism is developed in \cite{adapmoe_iccad,10906629}, where a transformer-based routing path predictor is trained offline to predict expert activation for each input token in a single pass. The predicted experts are then prefetched during the computation of previous layers. Furthermore, in \cite{he2024expertflow,cai2024textitreadme}, by decoupling the gating router from the MoE architecture, tokens with similar activation patterns can be grouped into micro-batches, further reducing fetching latency.

3) \textit{In-network expert caching:} Due to the huge sizes and computation burden, experts sometimes can be distributed over edge networks. By considering task relevance and wireless channel conditions, an energy-aware expert selection framework is developed, where experts are dispersed in wireless edge networks for distributed inference~\cite{qin2025optimal}. In this work, however, we focus on expert placement in edge networks based on the widely used Top-$K$ expert selection mechanism. To our best knowledge, the expert placement problem has not been investigated in the literature.

Expert placement is a key problem for in-network expert caching, i.e., which experts should be placed on which edge server. While caching has been extensively investigated in areas such as content delivery and model caching, expert caching in MoE models exhibits unique characteristics that render existing solutions inapplicable. In content caching, most prior work assumes independence among content items~\cite{10077798}, and the objective is typically to maximize a submodular utility function subject to matroid constraints~\cite{6600983}. This setting enables the application of classical greedy algorithms with a well-known $(1-1/e)$-approximation guarantee. In cases where dependencies exist, such as in caching interdependent tasks modeled by directed acyclic graphs, content placement decisions are often made sequentially, with each task requiring a single function to be cached and executed~\cite{10007839,10496897,9511632,10227271}.
In model caching, a recent approach, TrimCaching~\cite{10630945,qu2024trimcaching}, exploits parameter sharing across AI models to improve storage efficiency. By treating model layers as cacheable units, the caching problem can be formulated as a supermodular maximization problem under multiple knapsack constraints.
However, it assumes that a cached model is stored in its entirety, which is infeasible for large-scale MoE models due to their considerable parameter sizes. Expert caching in MoE models poses fundamentally different challenges due to the Top-$K$ expert activation mechanism, which introduces two key difficulties: 
1) \textit{Non-submodularity due to expert co-activation:} 
When $K>1$, multiple experts must be activated simultaneously for each token, leading to strong co-activation dependencies among experts. 
The dependency makes the expert caching problem violate both the submodularity and the supermodularity properties that are commonly exploited in existing caching formulations. As a result, existing methods cannot provide theoretical approximation guarantees for our expert placement problem.
2) \textit{Knapsack-type storage constraints:} 
In distributed caching problems, many prior studies assume that all items have identical sizes, which leads to simple cardinality or matroid-type constraints. In contrast, experts in practical MoE systems have heterogeneous parameter sizes across different models, and this heterogeneity induces multiple knapsack constraints on expert caching. Existing algorithms for nonsubmodular maximization do not provide approximation guarantees under multiple knapsack constraints. 
The above properties necessitate new problem formulations and optimization strategies for expert caching for distributed MoE inference.

	\begin{table}[ht]
		\caption{Summary of main notations} \label{tab:notations}
		\vspace{-10pt}
		\begin{center}
			{\footnotesize	\begin{tabular}{|c|p{4cm}|}\hline  
		\textbf{Notation} & \textbf{Definition} \\ \hline
		$ \mathcal{U}, \mathcal{N}, \mathcal{M} $&  Set of users, edge servers, and MoE models. \\ \hline
        $u, n, m$ & Index of user, edge server, and MoE model. \\ \hline
        $ \mathcal{L}_m, \mathcal{E}_m  $ & Set of MoE layers of model $m$, and the experts within each layer $\ell \in \mathcal{L}_m$. \\ \hline
        $K_m$   &  Top-$K_m$ expert activation strategy for model $m$.    \\ \hline
        $S_m^{(\ell)}, i_m^{(\ell)}$ & Expert(s) $S_m^{(\ell)}$ and a specific expert $i_m^{(\ell)} \in S_m^{(\ell)}$. \\ \hline
        $p_{u,m}, p_{u, S_m^{(\ell)}}$ &  Probability that user $u$ selects model $m$, and the probability that its token activates expert(s) $S_m^{(\ell)}$ in layer $\ell$ of model $m$.   \\ \hline
        $T_{u, S_m^{\left( \ell \right)  }}^{\mathrm{token}}, {\bar{T}}_{u,m}$ & Per-token inference latency when user $u$ requests expert(s) $S_m^{\left( \ell \right)} $. Average per-token inference latency when user $u$ requires model $m$. \\ \hline
        $r_{u,S_{m}^{(\ell)}}\left (\mathbf{X}   \right )$ & Latency reduction when user $u$ requires expert(s) $S_m^{\left( \ell \right)} $ under caching strategy $\mathbf{X}  $. \\ \hline
        $F\left( \mathbf{X} \right), {\tilde F}_n\left(\tilde{\mathbf{X}}_n\right), {\bar F}_n\left(\tilde{\mathbf{X}}_n\right)$  &  Original objective function of $\mathcal{P}1$. Objective function of edge server $n$ in subproblem $\mathcal{P}2_n$ given the caching decisions of edge servers 1 to $n-1$. 
        Transformed objective function of edge server $n$ in subproblem  $\mathcal{P}3_n$. \\ \hline
        $\left( \kappa_g\right)_n$ & Curvature of the supermodular term ${\tilde F}_{n}^{\mathrm{super} }$ in function $\tilde{F}_n\left ( \tilde{\mathbf{X}}_n  \right )$. \\ \hline
        $\ddot{\mathbf{X}}, \tilde{\mathbf{X}} ^{\ast }$ & Solution obtained by the proposed algorithm for $K \geq 1$ and the globally optimal solution of $\mathcal{P}1$. \\ \hline
			\end{tabular}}
		\end{center}
	\end{table}

	\section{System Model}\label{sec:system_model}
In this paper, we consider an MoE-based edge caching system with a set $\mathcal{N} = \left\lbrace 1, \ldots, N \right\rbrace $ of edge servers and a set $\mathcal{U} =  \left\lbrace 1, \ldots, U \right\rbrace $ of users.
The storage capacity of edge server $n \in \mathcal{N} $ is denoted by $Q_n$. 
Edge servers can directly communicate with each other via backhaul links and are also connected to the central cloud (expert model library) via a wired backhaul link. 
Each user connects to its nearest edge server and generates a task that requires an MoE architecture for inference. There are $M$ requested MoE models, and the set of models is denoted by $\mathcal{M} = \left\lbrace 1, \ldots, M \right\rbrace $. For model $m$, we consider the Top-$K_m$ expert selection mechanism, where $K_m$ is the number of activated experts within an MoE layer.
The main symbols and parameters used in this paper are summarized in Table~\ref{tab:notations} for clarity.

\subsection{Preliminaries: Expert Activation During MoE Inference}
We first briefly introduce the preliminaries of MoE inference. Fig.  \ref{fig:fig_MoE_procedure} shows the architecture of an MoE model.
	For MoE model $m \in \mathcal{M}$, the number of MoE layers is $L_m$, indexed by a set $\mathcal{L}_m = \left\{1, \ldots, L_m\right\}$. 
    Each MoE layer consists of a \textit{gate network} and a set of $E_m$ \textit{expert networks} (i.e., FFNs), and the data size of a single expert is denoted by $b_m $.  Thus, the total number of experts is given by $E = \sum_{m \in \mathcal{M}} L_m E_m$. We use $\mathcal{E}_m^{(\ell)}$ to denote the set of experts within model $m$'s MoE layer $\ell \in \mathcal{L}_m$.

	\begin{figure}[!t]
		\centering
		\includegraphics[width = 0.3\textwidth]{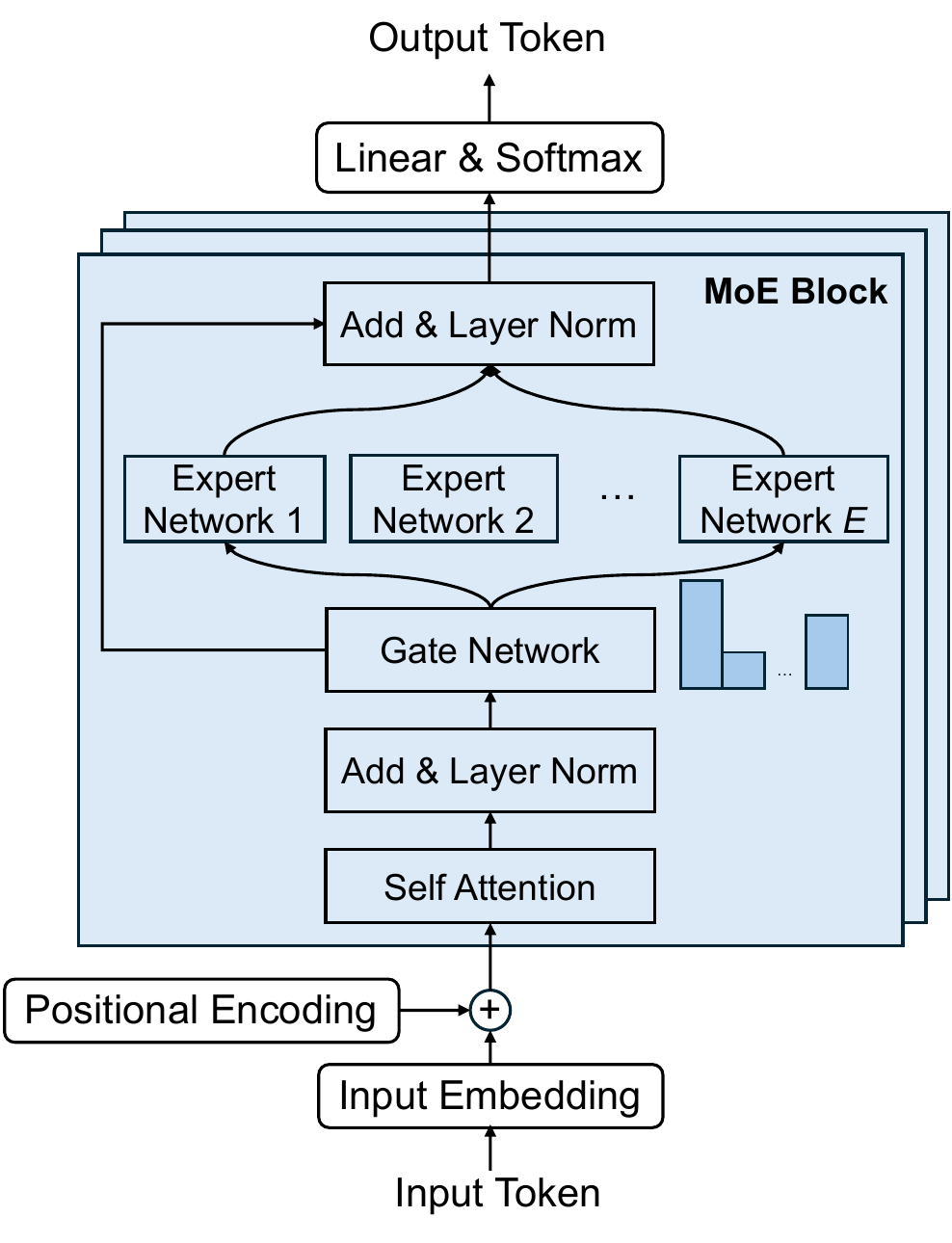}
		\caption{Illustration of an MoE architecture with $E$ expert networks in each MoE layer. \label{fig:fig_MoE_procedure}}
	\end{figure}

We take model $m$ as an example to illustrate how expert(s) are activated at the MoE layer $\ell$.
Let ${\mathbf{g}}\left ( \mathbf{s} \right )$ denote the token-to-expert affinity scores of an input $\mathbf{s}$, which represents the probability of the token selecting an expert for processing. Then, ${\mathbf{g}}\left ( \mathbf{s} \right )$ can be expressed as
\begin{equation}
    {\mathbf{g}}\left ( \mathbf{s} \right )   = \mathrm{Softmax}\left ( \mathbf{W}_r \mathbf{s}^{\mathrm{T} } \right ), 
\end{equation}
where $\mathbf{W}_r$ is the router parameter. With Top-$K_m$ routing strategy, the experts with $K_m$ highest scores will be activated. Then, the probability of expert $i$ can be normalized as follows:
	\begin{equation}
		\bar{{\mathbf{g}}} \left ( \mathbf{s}  \right )_{i} = \begin{cases}
			\frac{{\mathbf{g}}\left ( \mathbf{s} \right )_{i}}{\sum_{i' \in \mathrm{Top}K_m\left ( {\mathbf{g}}\left ( \mathbf{s}  \right ) \right )  }{\mathbf{g}}\left ( \mathbf{s} \right )_{i'}},   & \text{ if } i \in {\mathrm{Top}}K_m\left ( {\mathbf{g}}\left ( \mathbf{s}  \right ) \right ), \\
			0, & \text{ otherwise. }
		\end{cases}
	\end{equation}

Subsequently, the final output of the expert networks can be obtained by	
	${\mathbf{y}}\left ( \mathbf{s}  \right )   = \sum_{i \in {\mathcal{E}}_m^{(\ell)}} {\bar{\mathbf{g}} \left ( \mathbf{s}  \right )_{i} y_{i}\left ( \mathbf{s}  \right ) }$, where $y_{i}\left ( \mathbf{s}  \right ) $ is the output result of the selected expert $i$.

Let $\mathcal{A}_m^{(\ell)}$ denote the collection of all possible expert index sets of size $K_m$ that can be activated in MoE layer $\ell$ of model $m$, i.e., $\mathcal{A}_m^{\left( \ell\right) } = \left \{ S_m^{\left( \ell\right) } \subseteq \mathcal{E}_m^{\left( \ell\right) } \mid  \left | S_m \right | = K_m \right \} $.
For example, when $K_m=2$ and there are 3 experts in MoE layer $\ell$ of model $m$,  we have $\mathcal{A}_m^{\left( \ell\right) } = \left\{  \left\{1_m^{\left( \ell\right) },2_m^{\left( \ell\right) }   \right\} , \left\{1_m^{\left( \ell\right) },3_m^{\left( \ell\right) }   \right\}, \left\{2_m^{\left( \ell\right) },3_m^{\left( \ell\right) }   \right\} \right\}$ and one possible value of $S_m^{\left( \ell\right) }$ is $\left\{1_m^{\left( \ell\right) },2_m^{\left( \ell\right) }\right\}$.

	\subsection{Expert Caching Model}

	Given that expert networks account for the majority of MoE model's parameters, we assume that all components of MoE models except the experts are deployed on users, as shown in Fig. \ref{fig:system_model}. This strategy ensures raw input data and final predictions remain on the user side, thereby enhancing user privacy. Due to storage limitations, users can only store a limited number of experts in an MoE model. When the requested experts are not locally available, the user should send the hidden state of its token to edge servers with cached experts for inference. When the requested experts are not cached at any edge server, the user should send the hidden state to the cloud via its associated edge server, which is assumed to cache all experts. After the requested experts complete processing, the nodes return the output results to the user. Then, the processing of the next MoE layer begins. 
    
\begin{figure}[!t]
	\centering
	\includegraphics[width = 0.48\textwidth]{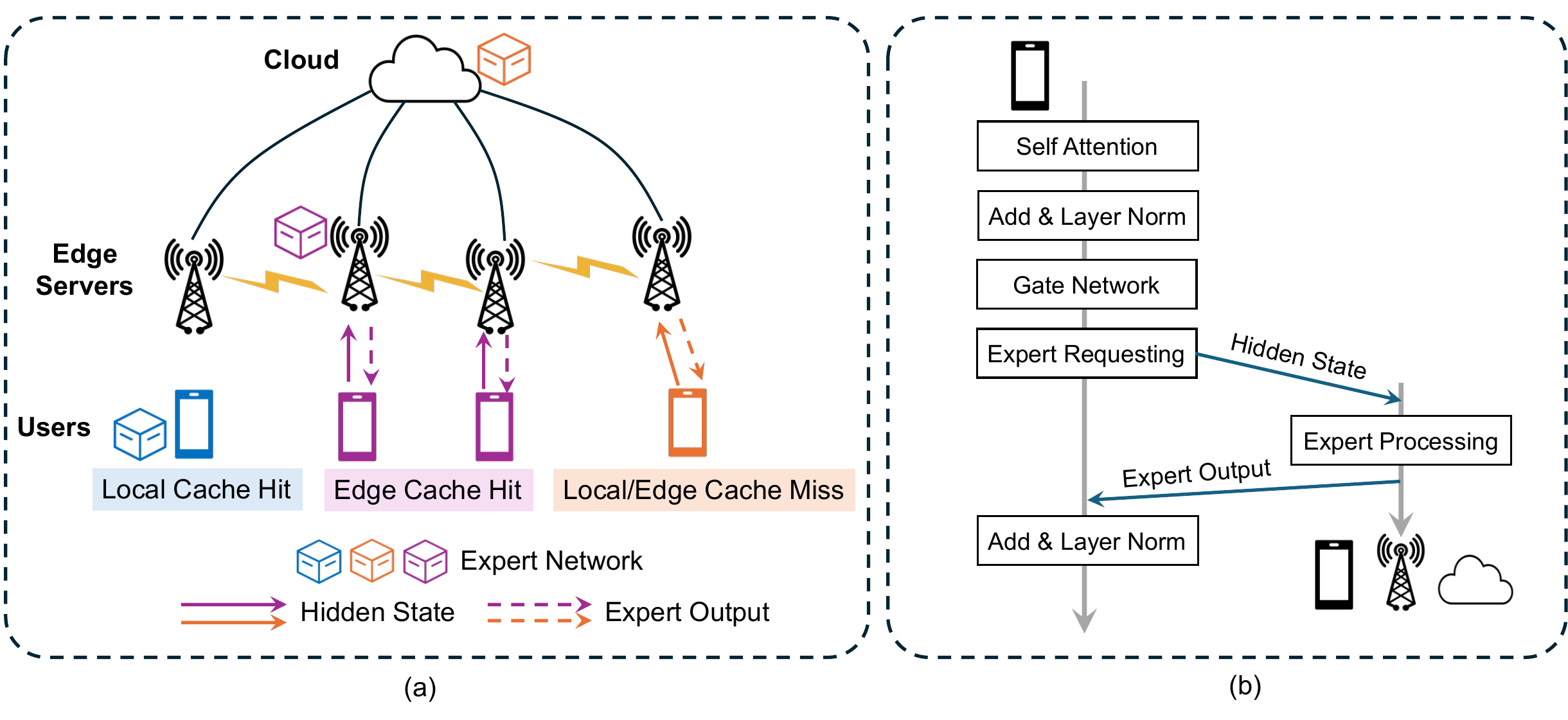}
	\caption{SlimCaching in distributed wireless systems. (a) Illustration of local cache hit, edge cache hit, and local/edge cache miss. (b) Operations of SlimCaching within an MoE layer.
    \label{fig:system_model}}
\end{figure}
    
    In this paper, we assume that the experts stored on users are predetermined based on user preferences (e.g., their most frequently used experts), whereas the placement of experts in edge networks should be optimized according to user requests and network conditions.
    Let $\rho_{u,i_m^{\left( \ell \right) }}$ denote the expert caching state of the users. If the expert with index $i$ of MoE model $m$'s layer $\ell $ is pre-cached on user $u$, $\rho_{u,i_m^{\left( \ell \right) }} = 1$ holds; otherwise, $\rho_{u,i_m^{\left( \ell \right) }} = 0$. Let $x_{n,i_m^{\left( \ell \right) }}$ denote the caching decision variable for the edge servers, with $x_{n,i_m^{\left( \ell \right) }} = 1$ indicating the expert with index $i$ of MoE model $m$'s layer $\ell $ is placed at edge server $n$ and $x_{n,i_m^{\left( \ell \right) }} = 0$ otherwise. 
    
	\subsection{Per-token Inference Latency}

To keep the problem tractable and to focus on the core structure of expert caching, this paper adopts a per-token formulation corresponding to the decoding phase, where one token is generated at each autoregressive step and expert routing occurs at this granularity.\footnote{Prefilling involves parallel processing of multiple tokens from the same user and may induce expert-level resource contention, such that per-token optimality does not directly translate to batch-level optimality. However, reducing per-token inference latency still reduces overall latency in the batch setting.} Next, we characterize the communication and computing latency needed to compute the per-token inference latency.

Let $R_{u,n}^{\rm{UL}}$ denote the expected data rate from user $u$ to edge server $n$, which can be expressed as
\begin{equation}\label{equ:datarate_UEserver_up}
		R_{u,n}^{\rm{UL}} = B_{u} \log_2 \left( 1 + \frac{P_{u,n} G_{u,n} d_{u,n}^{-\alpha}}{N_0 B_u} \right),
	\end{equation}
where $B_{u}$ is the bandwidth occupied by user $u$, $P_{u,n}$ is the transmit power from user $u$ to edge server $n$, $G_{u,n}$ is the uplink antenna-related factor, $d_{u,n}$ is the distance between user $u$ to edge server $n$, $\alpha$ is the path-loss coefficient, and $N_0$ is the noise power spectral density. Similarly, the downlink achievable rate from edge server $n$ to user $u$ is given by	
\begin{equation}
    R_{u,n}^{\rm{DL}} = B_{u} \log_2 \left( 1 + \frac{P_{n,u} G_{n,u}d_{u,n}^{-\alpha}}{N_0 B_u} \right),
\end{equation}
where $P_{n,u}$ is the transmit power from edge server $n$ to user $u$ and $G_{n,u}$ is the downlink antenna-related factor.
For the backhaul transmissions, let $R_{n,n'}$, $R_{n,\mathrm{C}}$, and $R_{\mathrm{C},n}$ denote the transmission rates from edge server $n$ to $n'$, from edge server $n$ to the cloud, and from the cloud to edge server $n$, respectively.

In this paper, we assume that each user is associated with the edge server that minimizes the sum of uplink and downlink transmission latency for transmitting one bit. Specifically, for user $u \in \mathcal{U}$, its associated edge server $n_u$ is determined as
\[n_u = \mathop{\arg\min}\limits_{n \in \mathcal{N}} \left \{ \frac{1}{R_{u,n}^{\rm{UL}}} + \frac{1}{R_{u,n}^{\rm{DL}}}  \right \}.\]

When performing MoE inference, the output tensor preserves the same dimensions as the input tensor. Thus, when using MoE model $m$ for task inference, the uplink and downlink have identical data transmission sizes, i.e., the size of hidden states of model $m$ denoted by $D_m$. When using model $m$ for inference, given the uplink data rate in (\ref{equ:datarate_UEserver_up}), the communication latency of transmitting a hidden state from user $u$ to its associated edge server $n_u$ is given by
\begin{equation}
    T_{u,n_u,m}^{\rm{UL}} = \frac{D_m}{R_{u,n_u}^{\mathrm{UL}}}.
\end{equation}

Similarly, the communication latency of transmitting a hidden state through the downlink between user $u$ and edge server $n_u$, and through the links from edge server $n$ to $n'$, from edge server $n$ to the cloud, and from the cloud to edge server $n$, can be expressed respectively as $T_{u,n_u,m}^{\rm{DL}} = \frac{D_m}{R_{u,n_u}^{\mathrm{DL}}}$, $ T_{n,n',m} = \frac{D_m}{R_{n,n'}}$, $T_{n,{\mathrm{C}},m} = \frac{D_m}{R_{n,{\mathrm{C}}}}$, and $T_{{\mathrm{C}},n,m} = \frac{D_m}{R_{{\mathrm{C}},n}}$, respectively.

Let $ \mu_{m}^{\mathrm{CP}} $ denote the computation workload (in FLOPs) of an expert in model $m$, which can be calculated according to \cite{abnar2025parameters}. Let 
$C_u$ be the computing capability (in FLOP/s) of user $u$, and the allocated computing capability for each expert of model $m$ denoted by ${\phi}_{u,m}$ is given by ${\phi}_{u,m} = C_u/E_m$.
Then, the local computing latency of user $u$ utilizing the expert network of model $m$ can be expressed as
	\begin{equation}
T_{u,m}^{{\mathrm{CP}}} = \frac{\mu_m^{\mathrm{CP}}}{{\phi}_{u,m}}.
	\end{equation}

Let ${\phi}_{\mathrm{E},m}$ and ${\phi}_{\mathrm{C},m}$ denote the allocated computing capability for each expert of model $m$ at the edge server and at the cloud, respectively. 
In the considered scenario, the users’ request traffic load is assumed to be lower than the maximum capacity of each edge server. This assumption allows each user’s token to be processed immediately upon arrival at the associated edge server, and inference access is thus modeled as parallel, without queueing delay.\footnote{A prior study has discussed the interplay between wireless communication resource allocation and learning or inference operations~\cite{10024766}. Our work focuses on expert placement under fixed communication parameters, and extending the framework to jointly incorporate communication resource allocation represents an important but orthogonal direction for future research.}
Thus, the computing latency of the edge server and the cloud utilizing the expert network of model $m$ can be given by  $T_{\mathrm{E},m}^{{\mathrm{CP}}} = \frac{\mu_m^{\mathrm{CP}} }{{\phi}_{\mathrm{E},m}}$ and $T_{\mathrm{C},m}^{{\mathrm{CP}}} = \frac{\mu_m^{\mathrm{CP}} }{{\phi}_{\mathrm{C},m}} $, respectively.

Consider that the bottleneck of MoE inference in distributed wireless networks is the communication overhead of inter-nodes, and the delay of token routing to the cloud is much longer than that to edge networks~\cite{10579139}. Thus, the order in which user $u$'s token searches for expert processing is as follows: local user, associated (nearest) edge server $n_u$, other non-associated edge servers $\mathcal{N} \setminus n_u$, and finally the cloud. Fig. \ref{fig:fig_4case} shows four different cases of token routing when a user requests the expert pair $(i_m^{\ell}, j_m^{\ell})$ under the Top-2 strategy.
\begin{figure}[t]
    \centering
\includegraphics[width=0.45\textwidth]{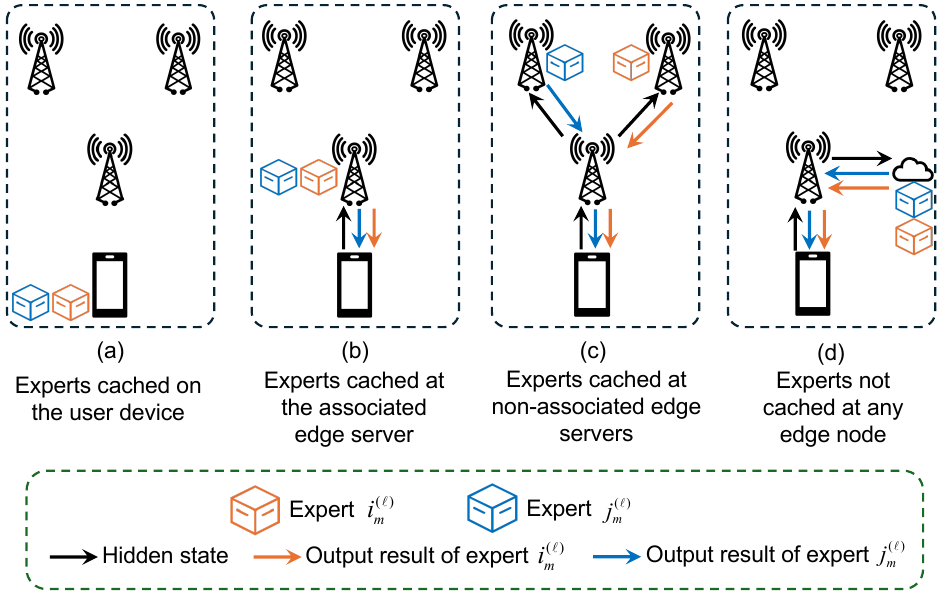}
    \caption{{Different cases of hidden-state routing when the user token’s hidden state activates the expert group $ \left\{ i_m^{\left( \ell \right)},j_m^{\left( \ell \right)}  \right\}$
    under the Top-2 strategy.}}
    \label{fig:fig_4case}
\end{figure}

Consider an arbitrary user $u$ who requires the expert (group) $ S_m^{\left( \ell \right)  }  $ in MoE layer $l$ of model $m$ for inference. Next, we will discuss the per-token inference latency when traversing this layer.
The number of activated experts at different nodes can be calculated as follows.
\begin{enumerate}
    \item At the local user $u$, there are $ \beta_{u,S_m^{\left( \ell\right) }} =  \mathop{\sum}\limits_{i_m^{\left( \ell\right) } \in S_m^{\left( \ell\right) }} \rho_{u,i_m^{\left( \ell\right) }} $ activated experts.

    \item At the associated edge server $n_u$, there are $\beta_{u,n_u,S_m^{\left( \ell\right) }} = \mathop{\sum}\limits_{i_m^{\left( \ell\right) } \in S_m^{\left( \ell\right) }} \left ( 1-\rho_{u,i_m^{\left( \ell\right) }}  \right ) x_{n_u,i_m^{\left( \ell\right) }}$ activated experts.

\item At the cloud, there are $\beta_{u,\mathrm{C} ,S_m^{\left( \ell\right) }} = \sum_{i_m^{\left( \ell\right) } \in S_m^{\left( \ell\right) }}\left ( 1-\rho_{u,i_m^{\left( \ell\right) }} \right ) \prod_{n \in \mathcal{N} } \left ( 1-x_{n,i_m^{\left( \ell\right) }} \right )$ activated experts.

    \item At other edge servers excluding the associated edge server,  there are $\beta_{u,S_m^{\left( \ell\right) }}^{\mathrm{OE}} = K_m - \beta_{u,S_m^{\left( \ell\right) }} -\beta_{u,n_u,S_m^{\left( \ell\right) }} - \beta_{u,\mathrm{C} ,S_m^{\left( \ell\right) }}$ activated experts in total.
\end{enumerate}

For the non-associated edge servers $n' \in \mathcal{N} \setminus n_u$, let $\beta_{u,n', i_m^{\left( \ell\right) } } \in \{0,1\}$ denote whether it provides processing with its cached expert $i_m^{\left( \ell\right) }$ for user $u$. 
Here, $\beta_{u,n', i_m^{\left( \ell\right) } } \leq x_{n',i_m^{\left( \ell\right) }}$ holds since edge server $n'$ can provide inference only if it caches the requested expert. Then, the total number of utilized experts at the non-associated edge server $n'$ is given by $\sum_{i_m^{\left( \ell\right) } \in S_m^{\left( \ell\right) }}\beta_{u,n', i_m^{\left( \ell\right) } }$, and $\beta_{u,S_m^{\left( \ell\right)}}^{\mathrm{OE}} = \sum_{n' \in \mathcal{N}\setminus n_u } \sum_{i_m^{\left( \ell\right) } \in S_m^{\left( \ell\right) }}\beta_{u,n', i_m^{\left( \ell\right) } } $ holds. Here, $\beta_{u,n', i_m^{\left( \ell\right) } } $ is determined in a way that leads to the lowest per-token inference latency, which can be easily obtained given expert placement decisions.

Considering bandwidth constraints, we assume that tokens are transmitted sequentially between the associated edge server $n_u$ and other edge servers. 
Accordingly, when user $u$ requests the expert (group) $S_m^{\left( \ell \right)}$, the total round-trip latency between $n_u$ and all other edge servers, denoted by $T_{n_u,S_m^{\left( \ell \right)}}^{\mathrm{OE} }$, can be expressed as
\begin{equation}\label{equ:latency_OE}
\begin{split}
        T_{n_u,S_m^{\left( \ell \right)}}^{\mathrm{OE} }  = \sum_{n' \in \mathcal{N}\setminus n_u } & \left[ {\mathbb{I}\left( {{\beta _{u,n',S_m^{\left( \ell  \right)}}} > 0} \right)\left( {{T_{{n_u},n',m}} + T_{{\mathrm{E}},m}^{{\mathrm{CP}}}} \right)} \right. \\
     & \left. + {\beta _{u,n',S_m^{\left( \ell  \right)}}}{T_{n',{n_u},m}} \right].
\end{split}
\end{equation}

In a special case where $K_m =1$, the token requires only one expert per MoE layer for processing. In this case, $S_m^{(\ell)}$ reduces to $i_m^{(\ell)}$ and $\beta_{u,n', S_m^{(\ell)}}
$ is given by
\begin{equation}\label{equ:closedform_beta_K1}
\begin{split}
        \beta_{u,n', S_m^{(\ell)}}
& =\left ( 1-\rho_{u,S_m^{\left( \ell\right) }}  \right )\left ( 1-x_{n_u,S_m^{\left( \ell\right) }} 
 \right ) \\
& \cdot \mathop{\prod}\limits_{\substack{n'' \in \mathcal{N} \\ T_{n_u,n'',m}+T_{n'',n_u,m} \\<T_{n_u,n',m}+T_{n',n_u,m}}} \left ( 1-x_{n'',S_m^{\left( \ell\right) }} \right ) x_{n',S_m^{\left( \ell\right) }}.
\end{split}
\end{equation}
That is, expert $S_m^{\left( \ell\right) }$ is cached at edge server $n'$, and all other edge servers $n'' \in \mathcal{N}$ that provide a lower round-trip transmission latency between $n_u$ and $n''$ (i.e., $T_{n_u,n'',m}+T_{n'',n_u,m}<T_{n_u,n',m}+T_{n',n_u,m}$) do not cache this expert. Therefore, the token of user $u$ is transmitted to edge server $n'$ for processing.

Let us get back to the general case. Based on (\ref{equ:latency_OE}), the latency experienced by a token at the edge servers, denoted by $T_{n_u, S_m^{\left( \ell \right)  }}^{\mathrm{E}}$, can be expressed as 
\begin{equation}
    T_{n_u, S_m^{\left( \ell \right)  }}^{\mathrm{E}} =  \mathbb{I} \left ( \beta_{u,n_u,S_m^{\left( \ell\right) }}>0 \right ) T_{\mathrm{E},m}^{\mathrm{CP} }+ T_{n_u,S_m^{\left( \ell \right)}}^{\mathrm{OE} }.
\end{equation}

In addition, the latency from the associated edge server $n_u$ to the cloud denoted by $T_{n_u, S_m^{\left( \ell \right)}}^{\mathrm{C} }$ can be expressed as
\begin{equation}
\begin{split}
    T_{n_u,S_m^{\left( \ell \right)}}^{\mathrm{C} } & = \mathbb{I}\left( {{\beta _{u,\mathrm{C} ,S_m^{\left( \ell  \right)}}} > 0} \right)\left( {{T_{{n_u},\mathrm{C},m}} + T_{{\mathrm{C}},m}^{{\mathrm{CP}}}} \right) \\
    & + {\beta _{u,\mathrm{C},S_m^{\left( \ell  \right)}}}T_{\mathrm{C},{n_u},m}.
\end{split}
\end{equation}

As a result, when the local experts cannot satisfy the inference requests, we can obtain the inference latency requiring outer nodes denoted by $T_{u, S_m^{\left( \ell \right)  }}^{\mathrm{out}}$ as follows,
\begin{equation}
\begin{split}
    T_{u, S_m^{\left( \ell \right)  }}^{\mathrm{out}} &= T_{u,n_u,m}^{\mathrm{UL} } + \left ( K_m-\sum_{i_m^{\left( \ell\right) } \in S_m^{\left( \ell\right) }}\rho_{u,i_m^{\left( \ell\right) }} \right )  T_{u,n_u,m}^{\mathrm{DL} }  \\
&+  T_{n_u, S_m^{\left( \ell \right)  }}^{\mathrm{E}}+ T_{n_u, S_m^{\left( \ell \right)  }}^{\mathrm{C}}.
\end{split}
\end{equation}

Finally, the per-token inference latency when user $u$ requests expert(s) $S_m^{\left( \ell \right) }$ for inference can be expressed as
\begin{equation}
	T_{u, S_m^{\left( \ell \right)  }}^{\mathrm{token}}    =   \prod _{i_m^{\left(\ell\right)} \in S_m^{\left( \ell\right) }}\rho_{u,i_m^{\left( \ell\right) }} T_u^{\rm{CP}}
+\left ( 1- \prod _{i_m^{\left(\ell\right)} \in S_m^{\left( \ell\right) }}\rho_{u,i_m^{\left( \ell\right) }} \right ) T_{u, S_m^{\left( \ell \right)  }}^{\mathrm{out}}.
\end{equation}

\subsection{Activation Probability and Average Latency}
As analyzed in Section \ref{sec:intro}, when using different MoE models, each token has a unique distribution of activation expert paths.
Let $p_{u,m}$ denote the probability that user $u$ selects model $m$ for inference.
Let $p_{u,S_{m}^{(\ell)}}$ denote the probability that user $u$ selects expert(s) $S_{m}^{(\ell)}$ in model $m$'s MoE layer $\ell$ for processing, which can be measured by historical statistics.
In practice, these probabilities are obtained from system-level observations and offline evaluation of the model’s decoding behavior. 
The model-request probability $p_{u,m}$ typically follows a Zipf-type popularity distribution and evolves slowly over time, while the expert-activation probability $p_{u,S_{m}^{(\ell)}}$ is estimated by running the MoE models on representative prompt traffic and recording empirical activation frequencies at each layer during the decoding phase. Both probabilities are generally stable in realistic systems. If drift occurs, they can be re-estimated using lightweight online statistics, allowing the caching policy to adapt over time.

Let ${\mathcal{S}}_{m,j} = \left \{ S_{m,j}^{(1)}, \dots, S_{m,j}^{(L_m)} \right \} \in \mathcal{S}_m $ denote the $j$-th path when selecting model $m$ for inference, where $S_{m,j}^{(\ell)} \in {\mathcal{S}}_{m,j}$ is the activated expert(s) in the $\ell$-th MoE layer. The probability of each individual path ${\mathcal{S}}_{m, j}$ for user $u$, denoted by $p_{u, {\mathcal{S}}_{m, j}}$, can be expressed as
\begin{equation}
\begin{split}
      p_{u, {\mathcal{S}}_{m, j}} &=p_u\left ( S_{m, j}^{(1)}, \dots, S_{m, j}^{(L_m)} \right )\\
      &= p_u\left ( S_{m, j}^{(1)} \right ) \prod_{\ell=2}^{L_m} p_u\left ( S_{m, j}^{(\ell)} \mid S_{m, j}^{(1)} , \ldots, S_{m, j}^{(\ell-1)}\right ),
\end{split}
\end{equation}
which is derived from the transition probability.

Let ${\bar{T}}_{u,m}$ denote the average per-token inference latency when user $u$ requires model $m$, which can be expressed as
\begin{equation}
    \begin{split}
        {\bar{T}}_{u,m}  &=    \sum_{{\mathcal{S}}_{m, j} \in \mathcal{S}_m} p_{u, {\mathcal{S}}_{m, j}} \sum_{S_{m,j}^{(\ell)} \in {\mathcal{S}}_{m,j}} T_{u,S_{m,j}^{(\ell)}}^{\mathrm{token}}  \\
& =  \sum_{\ell \in \mathcal{L}_m} \sum_{S_{m}^{(\ell)}\in {\mathcal{A} }_{m}^{(\ell)}} p_{u,S_{m}^{(\ell)}}T_{u,S_{m}^{(\ell)}}^{\mathrm{token}},
    \end{split}
\end{equation}
where the second equation is derived based on the characteristics of joint probability.

\section{Problem Formulation and Characterization}\label{sec:formulation}
In this section, we first formulate an optimization problem to minimize the average inference latency subject to the storage capacity constraints. Then, we identify and characterize the structural properties of the optimization problem under different $K$ values.

\subsection{Problem Formulation}
Let $T_{u, S_m^{\left( \ell \right)  }}^{\max} $ denote the maximal per-token inference latency when user $u$ requests the expert(s) $S_m^{\left( \ell \right)  }$. It is given by
\begin{equation}
    T_{u, S_m^{\left( \ell \right)  }}^{\max}    =   \prod _{i_m^{\left(\ell\right)} \in S_m^{\left( \ell\right) }}\rho_{u,i_m^{\left( \ell\right) }} T_u^{\rm{CP}}
+\left ( 1- \prod _{i_m^{\left(\ell\right)} \in S_m^{\left( \ell\right) }}\rho_{u,i_m^{\left( \ell\right) }} \right ) T_{u, S_m^{\left( \ell \right)  }}^{\mathrm{out},\max}.
\end{equation}

The term $ T_{u, S_m^{\left( \ell \right)  }}^{\mathrm{out,max}}$ represents the worst-case latency when the required experts that are not cached locally are also unavailable at any edge server, and thus the token must be processed at the cloud. $ T_{u, S_m^{\left( \ell \right)  }}^{\mathrm{out,max}}$ can be expressed as
\begin{equation}
\begin{split}
        T_{u, S_m^{\left( \ell \right)  }}^{\mathrm{out,max}} & = T_{u,n_u,m}^{\mathrm{UL} } + T_{n_u,{\mathrm{C}},m} + T_{\mathrm{C},m}^{\mathrm{CP}}   \\
        & +\left ( K_m-\sum_{i_m^{\left( \ell\right) } \in S_m^{\left( \ell\right) }}\rho_{u,i_m^{\left( \ell\right) }} \right ) 
\left ( T_{u,n_u,m}^{\mathrm{DL} } + T_{{\mathrm{C}},n_u,m} \right ).
\end{split}
\end{equation}

Then, the maximal average per-token inference latency when user $u$ requires model $m$ is given by $ T_{u,m}^{\max}  =  \sum_{\ell \in \mathcal{L}_m} \sum_{S_{m}^{(\ell)}\in {\mathcal{A} }_{m}^{(\ell)}} p_{u,S_{m}^{(\ell)}}T_{u, S_m^{\left( \ell \right)  }}^{\max}$.

To optimize expert placement for efficient MoE inference, we aim to minimize the average latency across all generated tokens. 
This objective is equivalently expressed as maximizing the reduction in the average per-token inference latency, leading to the following objective function:
\begin{equation}\label{equ:function_reduction}
\begin{split}
       F\left ( \mathbf{X}  \right ) & = \frac{1}{U} \sum_{u \in \mathcal{U}} \sum_{m \in \mathcal{M}} p_{u,m} \left ( T_{u,m}^{\max}-{\bar{T}}_{u,m}\left( \mathbf{X} \right)  \right )  \\
       &=\frac{1}{U} \sum_{u \in \mathcal{U}} \sum_{m \in \mathcal{M}} p_{u,m} \sum_{\ell \in \mathcal{L}_m} \sum_{S_{m}^{(\ell)}\in {\mathcal{A} }_{m}^{(\ell)}} p_{u,S_{m}^{(\ell)}} r_{u,S_{m}^{(\ell)}}\left ( \mathbf{X}  \right ),
\end{split}
\end{equation}
where the latency reduction when user $u$ requires expert(s) $S_{m}^{(\ell)}$ under the caching strategy $\mathbf{X}$ is defined as
$r_{u,S_{m}^{(\ell)}}\left (\mathbf{X}   \right ) \triangleq T_{u, S_m^{\left( \ell \right)  }}^{\max} - T_{u,S_{m}^{(\ell)}}^{\mathrm{token}}\left({\mathbf{X}}  \right)$.

Then, the optimization problem can be formulated as follows
\begin{subequations}
	\begin{equation}
\mathcal{P}1: \;		\max_{\mathbf{X}} \; F\left(\mathbf{X}\right) 
	\end{equation}
	\begin{equation} \label{cst:P1_knapsack}
		{\mathrm{s.t.}} \; \sum_{m \in \mathcal{M}}
		\sum_{\ell \in \mathcal{L}_m} \sum_{i_m^{\left(\ell\right)} \in \mathcal{E}_m^{\left(\ell\right)}} x_{n,i_m^{\left(\ell\right)}}b_m \leq Q_n, \; \forall n \in \mathcal{N},
	\end{equation}
   \begin{equation}\label{cst:P1_binary}
x_{n,i_m^{\left(\ell\right)}} \in \left\{0,1\right\}, \forall n \in  \mathcal{N}, m \in \mathcal{M}, 
        \ell \in \mathcal{L}_m,i_m^{(\ell)} \in \mathcal{E}_m^{(\ell)},
	\end{equation}
\end{subequations}
where Constraint (\ref{cst:P1_knapsack}) is a knapsack constraint ensuring that the cached experts within any edge server cannot exceed its storage capacity, and Constraint (\ref{cst:P1_binary}) denotes that the variable $\mathbf{X} = \left \{x_{n,i_m^{\left(\ell\right)}} \mid \forall n \in  \mathcal{N}, i_m^{(\ell)} \in \mathcal{E}_m^{(\ell)} \right \} $ is a set of binary variables.

\subsection{Problem Characterization}
Due to the cooperation among multiple edge servers and the correlation among activated experts, there is a complex coupling relationship within the decision variable, making $\mathcal{P}1$ challenging to solve.
In this subsection, we examine the properties of $\mathcal{P}1$ and perform problem mappings, providing a basis for the algorithm design in the next section.

We define the following ground set $V$:
$V = \{ \underbrace{v_1^1, \ldots, v_E^1}_{V_1}, \ldots, \underbrace{v_1^N, \ldots, v_E^N}_{V_N} \}$, where $v_i^n$ is an abstract element denoting the caching state of expert $i$ at the edge server $n$. 

\begin{definition}[Submodular Function]
    Defined over a finite set $V$, a set function $f: 2^V \rightarrow \mathbb{R}$ is said to be submodular when for every subset $X, Y \subseteq V$, we have $f \left( X\right) + f \left( Y\right) \geq f\left(X \cup Y\right) + f\left(X \cap Y\right)$. 
Alternatively, defining  $f \left( x | X \right) = f \left( X \cup x \right) -  f \left( X\right)$ as the marginal gain obtained by incorporating item $x$ into the set $X \subseteq V$, submodularity can equivalently be characterized by the \textit{diminishing returns} property: for any $X \subseteq Y \subset V$ and $v \in V \setminus Y$, it holds that $f(v | X) \geq f(v | Y)$. 
\end{definition}

\begin{definition}[Supermodular Function]
    A set function $f: 2^V \rightarrow \mathbb{R}$ is said to be supermodular if the function $-f$ is submodular. That is, for every $X, Y \subseteq V$, it holds that: $f(X) + f(Y) \leq f(X \cup Y) + f(X \cap Y)$.
    Equivalently, $f$ is supermodular if it satisfies the \textit{increasing returns} property: for all $X \subseteq Y \subseteq V$ and $v \in V \setminus Y$, $f(v \mid X) \leq f(v \mid Y)$ holds.
\end{definition}

\begin{proposition}\label{proposition:submodular_K1}
For Top-$K$ strategy, when $ K_m =1$ for any $m \in \mathcal{M}$, $\mathcal{P}1$ is a monotone non-decreasing submodular maximization problem with $N$ knapsack constraints.
\end{proposition}

\begin{proof}
Monotonicity is evident because adding an expert placement cannot reduce the objective value $F \left( \mathbf{X} \right)$. Moreover, as the caching set grows, the latency reduction brought by an additional expert decreases, yielding diminishing marginal returns. The complete proof is provided in Appendix \ref{proof:prop_submodular_K1}.

\end{proof}

\begin{proposition}\label{proposition:K2_nonsub_nonsuper}
    When $K_m >1$ for any $m \in \mathcal{M}$, $\mathcal{P}1$ is a monotone non-decreasing non-submodular non-supermodular
    maximization problem with $N$ knapsack constraints.
\end{proposition}
\begin{proof}
The monotonicity of $F \left( \mathbf{X} \right)$ is immediate and follows the same reasoning as the case $K_m = 1$.
To show that $F \left( \mathbf{X} \right)$ is neither submodular nor supermodular, we first present the key marginal behaviors under $K_m=2$ and subsequently generalize them to any $K_m >1$. The complete proof is given in Appendix \ref{proof:prop_K2_nonsub_nonsuper}.

\end{proof}

\section{Proposed Algorithms}\label{sec:proposed_alg}
Based on the structural analysis of $\mathcal{P}1$ under different values of $K$ in the previous section, we divide the algorithmic discussion into two parts: one for the special case of $K=1$, and the other for the general case of $K\geq 1$.
The special case $K = 1$ is suitable for the Switch Transformer (ST) family, such as ST-b-X with $X=\{8, 16, 32, 64, 128\}$ \cite{fedus2022switch} and Hash
Layer \cite{roller2021hash}. 
The general case $K \geq 1$ applies to 
GShard and MoE-LLaVA family \cite{11353922} with the Top-2 strategy and LLaMA-MoE-3.5B \cite{llama-moe} with the Top-4 strategy. In both cases, we will develop approximate algorithms with provable guarantees. Note that the algorithms for the general case $K \geq 1$ can obviously achieve an approximation guarantee to the special case $K = 1$, albeit with a slightly weaker approximation ratio compared to the algorithm designed for $K=1$.

\subsection{Special Case: $K=1$}
We denote the objective function of $\mathcal{P}1$ as $F_{1}\left ( \mathbf{X} \right )$ in the special case.
Let ${\mathbf{X}}^{\ast}$ denote the optimal solution of $F_{1}\left ( \mathbf{X} \right )$.

\subsubsection{Algorithm design}
Since the optimization problem $\mathcal{P}1$ is a submodular maximization problem with knapsack constraints, we can utilize the greedy-based algorithm to obtain the solutions, as shown in Algorithm \ref{alg:traditional_greedy}. Here, $F_{1}\left ( i \mid S \right )$ denotes the marginal value in the objective function $F_1$ when adding element $i$ to the set $S$.

\begin{algorithm}[t]
    \caption{\textsc{Greedy}-based Algorithm When $K=1$}\label{alg:traditional_greedy}
    \begin{algorithmic}[1]
        \State \textbf{Input:} Ground set $V=\left \{ V_1, \ldots,V_N \right \} $, objective function $F_1$, model $m$'s expert data size $b_m$, number of edge servers $N$, storage capacities $Q_n$ for $n \in \mathcal{N}$.
        \State \textbf{Output:} Expert caching results $S$ (and the way it is cached).
        \State 
        Initialize $S = \emptyset$.
        Let $S= \left \{ S_1, \ldots, S_N \right \}  $ denote its partitions for $N$ edge servers.
        \While{$V \setminus S \neq \emptyset$  and there exists $n \in \mathcal{N}$ such that $\mathop{\sum}\limits_{i \in S_n} b_i \leq Q_n$}
            \State $\left ( n^{*}, i^{*} \right )   =
\mathop{\arg\max}\limits_{\substack{n \in \mathcal{N}, i \in V_{n} \setminus S_{n}}}
\frac{F_{1}\left ( i \mid S \right ) }{b_i} $.
            \State $S = S \cup \left \{  i^*\right \} $ and $S_{n^{*}} = S_{n^{*}} \cup \left \{  i^*\right \} $.
        \EndWhile
        \State \Return $S= \left \{ S_1, \ldots, S_N \right \} $.
    \end{algorithmic}
\end{algorithm}

\begin{proposition}\label{prop:guarantee_traditional_greedy}
Algorithm \ref{alg:traditional_greedy} returns an expert caching set $S$ that satisfies $F_1\left ( S \right ) \geq \left ( 1- 1/e \right ) F_1\left ( {\mathbf{X}}^{\ast} \right )$.
\end{proposition}
\begin{proof}
    The proof follows the approach in \cite{DBLP:conf/esa/00010Z23}.
\end{proof}

\subsubsection{Computational complexity} 
Since Algorithm \ref{alg:traditional_greedy} is developed from greedy Algorithm, its computational complexity is given by \ref{alg:traditional_greedy} is $\mathcal{O}(N^2E^2)$.

\subsection{General Case: $K \geq 1$}
In this subsection, we extend our discussion to the general case where both $K=1$ and $K>1$ strategies are involved.
Recall that for $K >1$, the objective function of $\mathcal{P}1$ is neither submodular nor supermodular, making it challenging to solve. To address this, we first decompose the original problem $\mathcal{P}1$ into $N$ subproblems and solve them successively in the order of edge server index. Based on the analysis of each subproblem, we demonstrate why the greedy algorithm—effective in the special case—fails to provide a constant approximation ratio in the general case. Leveraging the structure of each subproblem, we then propose a dynamic programming (DP)-based algorithm and also an accelerated algorithm to obtain solutions with a constant approximation guarantee. A detailed discussion of the proposed algorithms and their performance is also provided.

\subsubsection{Problem decomposition}
The subproblem of edge server $n$ can be expressed as follows.
\begin{subequations}
	\begin{equation}
        \mathcal{P}2_n: \;		\max_{\tilde{\mathbf{X}}_n} \;  {\tilde F}_n\left(\tilde{\mathbf{X}}_n\right) 
	\end{equation}
	\begin{equation} \label{cst:user_knapsack}
		{\mathrm{s.t.}} \; \sum_{m \in \mathcal{M}}
		\sum_{\ell \in \mathcal{L}_m} \sum_{i \in \mathcal{E}_m} x_{n,i_m^{\left(\ell\right)}}b_m \leq Q_n, 
	\end{equation}
	\begin{equation} \label{cst:user_binary}
		x_{n,i_m^{\left(\ell\right)}} \in \left\{0,1\right\}, \; \forall m \in \mathcal{M},
\ell \in \mathcal{L}_m,i_m^{(\ell)} \in \mathcal{E}_m^{(\ell)},
	\end{equation}
\end{subequations}
The objective function ${\tilde F}_n\left(\tilde{\mathbf{X}}_n\right)$ is given by 
\begin{align}
    {\tilde F}_n\left(\tilde{\mathbf{X}}_n\right) =  \sum_{u \in \mathcal{U}} \sum_{m \in \mathcal{M}} p_{u,m} \sum_{\ell \in \mathcal{L}_m} \sum_{S_{m}^{(\ell)}\in {\mathcal{A} }_{m}^{(\ell)}} p_{u,S_{m}^{(\ell)}}{\tilde{r}}_{u,S_{m}^{(\ell)}}\left (\tilde{\mathbf{X}}_n  \right ),
\end{align}
where
\begin{align}
\resizebox{1\hsize}{!}{$   {\tilde{r}}_{u,S_{m}^{(\ell)}}\left (\tilde{\mathbf{X}}_n  \right )  = T_{u,S_{m}^{(\ell)}}^{\mathrm{token}}\left( \bigcup_{n'=1}^{n-1} \tilde{\mathbf{X}}_{n'} \right) - T_{u,S_{m}^{(\ell)}}^{\mathrm{token}}\left( \tilde{\mathbf{X}}_n \cup \left( \bigcup_{n'=1}^{n-1} \tilde{\mathbf{X}}_{n'} \right) \right). $}
\end{align}
That is, given the caching results of the previous edge servers from $n=1$ to $n-1$, we aim to optimize the caching strategy of edge server $n$ to maximize the average reduction in latency.
Specifically, when $n=1$, ${\tilde{r}}_{u,S_{m}^{(\ell)}}\left (\tilde{\mathbf{X}}_1  \right )  = T_{u,S_{m}^{(\ell)}}^{\max} - T_{u,S_{m}^{(\ell)}}^{\mathrm{token}}\left( \tilde{\mathbf{X}}_1 \right)$.
Therefore, the objective function of $\mathcal{P}1$ follows from 
\begin{align}\label{equ:union_solution}
    F(\tilde{\mathbf{X}}) = F\left( \bigcup_{n \in \mathcal{N}} \tilde{\mathbf{X}}_n \right) = \sum_{n \in \mathcal{N}} \tilde{F}_n\left(\tilde{\mathbf{X}}_n  \right) .
\end{align}

The second equality holds since the sum of $\tilde{F}_n$ is a telescoping sum, where all intermediate terms cancel out, leaving only the first and the last terms.
This motivates us to divide the original problem into $N$ subproblems and solve them successively.

\subsubsection{Subproblem analysis for $K\geq 1$}
Take the subproblem $\mathcal{P}2_n$ as an example, we analyze its property and propose an efficient accelerated knapsack algorithm to solve it.

\begin{proposition}\label{proposition:subproblem_modularplussuper}
$\mathcal{P}2_n$ is maximizing the sum of a non-decreasing modular function and a non-decreasing supermodular function subject to a knapsack constraint. The modular function corresponds to the term including $m$ with $K_m = 1$ and the supermodular function corresponds to the term including $m$ with $K_m > 1$. 
\end{proposition}

\begin{proof}
    The proof of Proposition \ref{proposition:subproblem_modularplussuper} is omitted here because the procedure is similar to that of Proposition \ref{proposition:K2_nonsub_nonsuper}. 
\end{proof}

First, we provide an explanation about why the traditional greedy algorithm (Algorithm \ref{alg:traditional_greedy}) cannot achieve a constant-approximation guarantee when solving $\mathcal{P}2_n$.

\begin{remark}\label{remark:greedy_limitation}
(\textit{Limitations of Greedy Methods in Supermodular Optimization Problems})
In MoE inference with $K > 1$, each token must activate multiple experts simultaneously within a layer, introducing strong interdependencies among the selected experts. The overall utility of a selection can only be accurately assessed once all required experts have been determined. Prior to this, the contribution of any individual expert is difficult to evaluate in isolation, often resulting in an underestimation of its actual utility.
While the greedy algorithm may be effective when $K = 1$, where expert activations are independent and the utility of each expert can be evaluated individually, it fails to account for the combinatorial nature of utility when $K > 1$. In these cases, early greedy selections are made with incomplete information and may restrict the effectiveness of subsequent choices, leading to suboptimal overall performance due to cascading effects.

\end{remark}

\subsubsection{The proposed DP-based algorithm}
Since the greedy algorithm is improper to solve $\mathcal{P}2_n$, we reformulate $\mathcal{P}2_n$ into a tractable modular optimization problem and provide solutions.

Let $\tilde{F}_n\left ( \tilde{\mathbf{X}}_n  \right )  \triangleq {\tilde F}_n^{\mathrm{mod} } \left ( \tilde{\mathbf{X}}_n  \right )+ {\tilde F}_n^{\mathrm{super} } \left ( \tilde{\mathbf{X}}_n  \right )$, where
${\tilde F}_n^{\mathrm{mod} }$ and
${\tilde F}_n^{\mathrm{super} }$ denote the modular and supermodular components of ${\tilde F}_n$, respectively. 
We define a modular function as ${H_n} : 2^{\mathcal{N}} \to \mathbb{R}_+$ with $H_n \left ( \tilde{\mathbf{X}}_n \right )  = \sum_{z \in \tilde{\mathbf{X}}_n} {\tilde F}_n^{\mathrm{super} }(z \mid \emptyset)$. 
Then, we reformulate $\mathcal{P}2_n$ as $\mathcal{P}3_n$, which can be expressed as follows,
\begin{subequations}
    \begin{equation}
        \mathcal{P}3_n: \mathop{\max}\limits_{\tilde{\mathbf{X}}_n} \;   {\bar F}_n  \left ( \tilde{\mathbf{X}}_n \right ) =     {\tilde F}_n^{\mathrm{mod} }  \left ( \tilde{\mathbf{X}}_n \right ) + H_n \left ( \tilde{\mathbf{X}}_n \right )
    \end{equation}
    \begin{equation}
        {\rm{s.t.} \; (\ref{cst:user_knapsack}), (\ref{cst:user_binary})}.
    \end{equation}
\end{subequations}

Since the objective function ${\bar F}_n$ is a modular function, $\mathcal{P}3_n$ is a general knapsack problem, which can be solved optimally with a DP-based algorithm in polynomial time. We outline the procedure in Algorithm \ref{alg:general_K>=1}. 
The gap between the proposed solutions and the globally optimal solutions of $\mathcal{P}1$ will be discussed in the following paragraphs.

\begin{algorithm}[t]
\caption{Proposed DP-based Algorithm When $K \geq 1$}
\label{alg:general_K>=1}
\begin{algorithmic}[1]
\Require Expert values ${\tilde F}_n^{\mathrm{mod} }(i \mid \emptyset)$ or ${\tilde F}_n^{\mathrm{super} }(i \mid \emptyset)$ for $i \in V_n$, expert data size $b_{m(i)}$ for $i \in V_n$, edge server $n$'s storage capacity $Q_n$, $V_n$'s cardinality $E$.
\Ensure Expert caching results $S$ (and the way it is cached).
\For{$n \in \mathcal{N} $}
\State Initialize ${\rm{dp}}[0 \ldots E][0 \ldots Q_n] \gets 0$
\For{$i \in [1, E] $}
    \For{$q \in [0, Q_n] $}
        \If{$b_{m(i)} > q$}
            \State ${\rm{dp}}[i][q] \gets {\rm{dp}}[i-1][q]$
        \Else
        \If{$K_{m(i)} = 1$}
            \State ${\rm{dp}}[i][q] \gets \max({\rm{dp}}[i-1][q],\ {\rm{dp}}[i-1][q - b_{m(i)}] + {\tilde F}_n^{\mathrm{mod} }(i \mid \emptyset))$
            \Else
            \State ${\rm{dp}}[i][q] \gets \max({\rm{dp}}[i-1][q],\ {\rm{dp}}[i-1][q - b_{m(i)}] + {\tilde F}_n^{\mathrm{super} }(i \mid \emptyset))$
        \EndIf
        \EndIf
    \EndFor
\EndFor
\State \Return ${\rm{dp}}[E][Q_n]$
\State // Traceback to get selected experts of edge server $n$
\State Initialize $\ddot { \mathbf{X} }_n \gets \emptyset$, $i \gets E$, $q \gets Q_n$
\While{$i > 0$}
    \If{${\rm dp}[i][q] \ne {\rm dp}[i-1][q]$}
        \State $\ddot { \mathbf{X} }_n \gets \ddot { \mathbf{X} }_n \cup \{i\}$, $q \gets q - b_{m(i)}$
    \EndIf
    \State $i \gets i - 1$
\EndWhile
\State \Return $\ddot { \mathbf{X} }_n$
\EndFor
\State \Return  $\ddot { \mathbf{X} }= \left \{ \ddot { \mathbf{X} }_1, \ldots, \ddot { \mathbf{X} }_N \right \} $
\end{algorithmic}
\end{algorithm}

\subsubsection{The proposed accelerated algorithm}
Let $T$ denote the distinct data size of experts among all the MoE models.
We notice that $T \leq M \ll E$. Here, $\leq$ lies in that numerous downstream models are developed by fine-tuning a pre-trained base model, and $\ll$ comes from the same data size of experts in each MoE model.
Based on Definition \ref{def:convolution}, we design Algorithm \ref{alg:knapsack_sameweight} with lower complexity than the traditional DP algorithm.

\begin{definition}\label{def:convolution}
Let $y^{(i)}$ represent the subsequence of $a$ formed by selecting elements at positions congruent to $i \mod k$, that is, $y^{(i)} = (a_i, a_{i + k}, a_{i + 2k}, \dots)$. Let $z$ represent a $k$-step concave sequence $z=(b_0, b_k, b_{2k}, \ldots)$.
For any $q$, we can compute the (max,+) convolution of $y^{(i)}$ and $z$ as follows:
\[
\left( y^{(i)} \oplus z \right)_q = \max_{l=0}^{\infty} \{ y^{(i)}_{q-l}+z_l   \} 
= \max_{l=0}^{\infty} \left \{ a_{i+qk - lk} + b_{lk}  \right \}.  \]
\end{definition}

\begin{algorithm}[t]
    \caption{Proposed Accelerated Algorithm When $K \geq 1$}\label{alg:knapsack_sameweight}
    \begin{algorithmic}[1]
        \State \textbf{Input:} Experts with data size in $\{\tilde{b}_1, \ldots, \tilde{b}_T\}$, expert values ${\tilde F}_n^{\mathrm{mod} }(i \mid \emptyset)$ or ${\tilde F}_n^{\mathrm{super} }(i \mid \emptyset)$ for $i \in V_n$, storage capacity $Q_n$
        \For{$n \in \mathcal{N} $}
        \State Partition items into sets $E_1, \ldots, E_T$, such that $E_i = \{j \mid b_j = {\tilde{b}}_i\}$
        \For{$i \in \{1, \ldots, T\}$ and $q \in \{1, \ldots, Q_n\}$}
            \State $s_q^{(i)} \gets$ solution for $E_i$ with knapsack size $q$
        \EndFor
        \State Initialize $\ddot { \mathbf{X} }_n$ as an empty sequence
        \For{$i \in \{1, \ldots, T\}$}
            \State $\ddot { \mathbf{X} }_n \gets \ddot { \mathbf{X} }_n \oplus s^{(i)}$ using Definition \ref{def:convolution}
        \EndFor
        \State Update $\ddot { \mathbf{X} }_n$ by only keeping its first $Q_n$ elements 
        \EndFor
        \State \Return  $\ddot { \mathbf{X} }= \left \{ \ddot { \mathbf{X} }_1, \ldots, \ddot { \mathbf{X} }_N \right \} $
    \end{algorithmic}
\end{algorithm}

To improve efficiency without sacrificing optimality, Algorithm \ref{alg:knapsack_sameweight} leverages the observation that experts can naturally be grouped by their data sizes into a small number of categories. Each subset induces a subsequence whose contributions to the overall solution can be computed independently via efficient convolution operations, as defined in Definition \ref{def:convolution}. These $T$ (satisfying $T \ll E$) subsequences preserve the structure required by the original DP formulation, allowing their partial solutions to be combined without loss of optimality.

\subsubsection{Computational complexity}
The computational complexity of Algorithm \ref{alg:general_K>=1} is $\mathcal{O}(NE\left \lceil Q_n / (b_m)_{\min} \right \rceil)$, while the computational complexity of Algorithm \ref{alg:knapsack_sameweight} $\mathcal{O}(NT\left \lceil Q_n / (b_m)_{\min} \right \rceil))$. The detailed proof procedure is referred to \cite{axiotis_et_al}. 

Compared to the computational complexity of Algorithm \ref{alg:traditional_greedy}, the computational complexity of our proposed algorithm is lower since $T$ is much lower than $E$ in MoE models and $\left \lceil Q_n / (b_m)_{\min} \right \rceil)$ is also lower than $E$, especially in the storage-constrained edge networks.

\subsubsection{Approximation guarantee}\label{subsub:approximation_guarantee}
In this part, we will provide the approximation guarantee obtained by Algorithm \ref{alg:knapsack_sameweight} when solving each subproblem $\mathcal{P}2_n$ and then give the approximation guarantee of the global solution.

First, the definition of supermodular curvature is provided, followed by the property of supermodular functions.

\begin{definition} \label{def:super_curv}
    (\textit{Supermodular Curvature} \cite{bai2018greed}) 
    The supermodular curvature of a monotone non-decreasing non-negative supermodular function $g : 2^{V} \rightarrow \mathbb{R}$ is defined as 
    $\kappa_{g} = 1 - \min_{z \in V} \frac{g(z \mid \emptyset)}{g(z \mid V \setminus  z)}$.
\end{definition}

Supermodular curvature $\kappa_g$ is computationally feasible and requires only linear time in the oracle model. Furthermore, we have the following conclusions.

\begin{lemma}\label{lemma:property_super}
(Property of the supermodular function \cite[Lemma 2]{Lu2022NonSubmodular})
If $g : 2^{V} \rightarrow \mathbb{R}$ is a non-decreasing supermodular function with curvature $\kappa_g$, then for every $A \subseteq V$, it holds that $(1 - \kappa_g)g\left ( A \mid \emptyset \right )  \leq \sum_{z \in A} g(z \mid \emptyset) \leq g\left ( A \mid \emptyset \right )$.

\end{lemma}

\begin{proposition}\label{prop:subproblem_ratio}
Let $\ddot { \mathbf{X} }_n \subseteq  V$ denote the optimal solution of $\mathcal{P}3_n$ obtained by Algorithm \ref{alg:knapsack_sameweight}. 
We have the following approximation ratio,
\begin{align}\label{equ:lower_bound_subproblem}
    {\tilde F}_n\left ( \ddot { \mathbf{X} }_n \right )  \geq \left ( 1-(\kappa_g)_n \right ) {\tilde F}_n\left ( \tilde{\mathbf{X}}_n^{\ast} \right ),
\end{align}
where $\tilde{\mathbf{X}}_n^{\ast} \in \arg\max \left\{ {\tilde F}_n\left ( \tilde{\mathbf{X}}_n \right ): \tilde{\mathbf{X}}_n \subseteq V,\ {(\ref{cst:user_knapsack}), (\ref{cst:user_binary})} \right\}$ and $(\kappa_g)_n$ is the curvature of the supermodular term ${\tilde F}_{n}^{\mathrm{super} }$ in function $\tilde{F}_n\left ( \tilde{\mathbf{X}}_n  \right )$.

\end{proposition}

\begin{proof}
The result is obtained by lower-bounding the supermodular component via the generalized curvature and then comparing the resulting modularized objective with the optimal value. The detailed proof is provided in Appendix \ref{proof:prop_subproblem_ratio}.

\end{proof}

Finally, we establish the following theorem, providing a constant approximation guarantee for our proposed algorithm.

\begin{theorem}\label{theorem:approximation_ratio_general_case}
Let $ \ddot{\mathbf{X} } $ and $ \tilde{\mathbf{X}} ^{\ast }$  denote the solution obtained by the proposed algorithm for $K \geq 1$ and the globally optimal solution of $\mathcal{P}1$.
 We have $F\left ( \ddot{\mathbf{X} }  \right ) \geq 
\frac{1-\kappa_g^{\max}}{2}  F\left ( \tilde{\mathbf{X}} ^{\ast }  \right )$, where $\kappa_g^{\max}=\max_{n \in \mathcal{N} } \; (\kappa_g)_n$. 
Specifically, when $N=1$, the obtained solutions by the proposed algorithm can provide $(1-(\kappa_g)_1)$-approximation guarantee.
\end{theorem}

\begin{proof}
By applying Proposition \ref{prop:subproblem_ratio} to each edge server and introducing a comparison set that isolates the optimal per-edge server contribution, we establish two bounds that relate the optimal solution, the comparison set, and the solution returned by the proposed algorithm. The detailed proof is provided in Appendix \ref{proof:theorem_approximation_ratio_general_case}.

\end{proof}

\subsection{Discussion of Proposed Algorithm and Solutions}

In this subsection, we refine the previously established approximation guarantee in Theorem \ref{theorem:approximation_ratio_general_case} under some practical assumptions, discuss the applicability of the proposed algorithms in special model-storage scenarios, and further explain why the proposed method performs better over the greedy algorithm.

\begin{proposition}[Approximation Ratio Under Symmetric Links and Communication-dominant Latency]~\label{prop:approximationratio_condition}
We consider a setting where communication links among edge servers and between each edge server and the cloud are symmetric. We also assume that communication latency dominates the E2E inference latency, rendering computing latency negligible.
Under these assumptions, the curvature characterization in Theorem \ref{theorem:approximation_ratio_general_case} leads to the following guarantees.
In the single-edge-server scenario, the curvature satisfies $\left ( \kappa_g \right )_1 \approx \frac{1}{2}$, and the proposed algorithm achieves a $\frac{1}{2}$-approximation ratio.
In the general multi-edge-server scenario, the curvature lies in the interval $\kappa_g^{\max} \in (\frac{1}{2}, 1)$, and the proposed algorithm guarantees a $\frac{1}{4}$-approximation ratio.
\end{proposition}

\begin{proof}
	
Please see Appendix \ref{proof:prop_approximationratio_condition}.

\end{proof}

\begin{remark}[Applicability of Algorithm \ref{alg:general_K>=1} and \ref{alg:knapsack_sameweight} in Special Model Storage Cases]
While Algorithm \ref{alg:general_K>=1} and \ref{alg:knapsack_sameweight} are designed for $K \geq 1$, they can also be applied to the special case $K_m=1$ for any $m \in \mathcal{M}$. If we use Algorithm \ref{alg:general_K>=1} or \ref{alg:knapsack_sameweight} to successively solve $\mathcal{P}2_n$, the obtained solutions can provide $1/2$-approximation guarantee. The reason lies in the fact that only the modular part is left when solving $\mathcal{P}2_n$ for any $n \in \mathcal{N}$, and the optimal solution can be obtained using Algorithm \ref{alg:general_K>=1} or \ref{alg:knapsack_sameweight}.
\end{remark}

\begin{remark}[Why the Proposed Algorithm Outperforms Greedy Algorithm]
   The proposed method precomputes a utility score for each expert based on its isolated contribution when added to an empty set. Although the utilities are computed independently, the dynamic programming approach systematically explores the combinatorial space of expert subsets. This strategy mitigates the shortsightedness inherent in greedy methods and yields higher-quality solutions.
\end{remark}

\section{Experimental Results}\label{sec:experiment}
In this section, we conduct simulations to demonstrate the effectiveness of the proposed algorithms.

\subsection{Experimental Settings}
In our simulations, there are $N=4$ edge servers and $U=20$ edge devices distributed in a cell of size 1 km $\times$ 1 km.
The wireless parameters are set as follows: The transmit power of users and edge servers is 0.01 W~\cite{9210812} and 38 dBm~\cite{3GPP_TR_25_951}, respectively. The bandwidth between a user and its associated edge server is 5 MHz and the bandwidth between edge servers is 100 MHz. The latency between edge servers and cloud is 0.01 s~\cite{8463568}.
The path-loss coefficient is $\alpha = 4$, the antenna-related factor is 1, and the noise power spectral density is -174 dBm/Hz~\cite{qu2024trimcaching}.
The computing capabilities of edge devices, edge servers, and the cloud are 50 TFLOPs, 82.58 TFLOPs, and 312 TFLOPs, respectively. 
Besides, the storage capacity of edge servers for expert caching is 5 GB.
Our MoE model library contains 18 models with a total of nearly 4,000 experts. 
These models can be grouped into three families. The first family is the Switch Transformer series, including ST-b-8, ST-b-16, and ST-b-32, which all adopt a Top-1 expert activation strategy. The second family is the MoE-LLaVA series, consisting of MoE-LLaVA-StableLM-1.6B-4e, MoE-LLaVA-Qwen-1.8B-4e, and MoE-LLaVA-Phi2-2.7B-4e, which use a Top-2 activation policy. The third family is the LLaMA-MoE series, exemplified by LLaMA-MoE-3.5B, which utilizes a Top-4 expert activation mechanism.
Each user has a set of 3 to 5 models to request, and each time it requests one model drawn from a Zipf distribution. Given the limited storage capacity of users and the assumption that the non-expert components of the MoE models are stored locally, each user stores 200 experts from the entire set of experts for the requested models.
The MoE models are tested on the SQA~\cite{lu2022learn} and VQA-v2~\cite{goyal2017making} datasets.
All experiments were conducted in two stages. In the offline stage, we computed the expert activation statistics for all tokens across different MoE models and datasets using an NVIDIA GeForce RTX 4090 GPU. In the online stage, the proposed SlimCaching algorithm and all baseline methods were implemented and executed in MATLAB R2024b on a machine equipped with an Apple M3 processor.

We adopt the following algorithms as benchmarks:
\begin{itemize}
    \item \textbf{Greedy Algorithm}: In each step, the algorithm greedily selects the expert and its corresponding edge server that provides the maximum marginal gain in utility, which corresponds to Algorithm \ref{alg:traditional_greedy}.

    \item \textbf{Least Frequently Used (LFU) Algorithm}: Each edge server caches the most frequently requested experts until reaching the upper limit of its storage capacity. 

        \item \textbf{Random Algorithm}: Each edge server randomly selects experts to cache until reaching the upper limit of its storage capacity.

        \item \textbf{U-shaped (layer)}:
        This baseline follows a U-shaped SI design and caches parameters at the granularity of an entire MoE block rather than at the expert level. Once an MoE layer is selected for caching at an edge server, all experts within that layer are stored together until the storage capacity is reached.
\end{itemize}

Overall, the first three benchmark algorithms perform caching at the expert granularity, while the ``U-shaped (layer)" baseline follows a layer-level policy that caches entire MoE blocks as indivisible units, providing a contrasting distributed inference paradigm for comparison.

\subsection{Performance Comparison}
This subsection evaluates the latency performance of different algorithms by varying the storage capacities of edge servers $Q$, the number of experts locally stored by each user $\rho$, the distribution of user requests for MoE models $p_{u,m}$, the bandwidth between the user and its associated edge server $B_u$, and the number of edge servers $N$ and users $U$. 

Fig. \ref{fig:fig_storage_BS} presents the average per-token latency under different storage capacities at the edge servers. Across all tested configurations, the proposed algorithm consistently achieves the lowest latency.
For instance, under a constrained storage budget of 2.5 GB, the proposed method achieves an average latency of 10.34 ms, yielding a 16.7\% reduction compared with the greedy scheme and a 19.5\% reduction compared with the LFU scheme.
As analyzed in Remark \ref{remark:greedy_limitation}, the greedy scheme suffers from locally optimal caching choices, while LFU treats each edge server independently. In contrast, our successive-greedy decomposition explicitly captures the cooperative caching structure among edge servers, enabling more effective expert placement at the global level.
Both the Random and U-shaped schemes result in substantially higher latency compared with the above three expert-based caching schemes across all storage configurations. 
In particular, the U-shaped approach exhibits high latency because the hidden-state transmission pattern of an MoE model is fixed by whether the full model is cached across users and edge servers. If a user does not store all MoE layers, the hidden state of every token must be uploaded to either the edge server or the cloud for processing. Furthermore, if a user and edge servers together still do not hold the complete set of MoE layers, every token must inevitably be sent to the cloud. As a result, increasing the storage capacity of the edge server has little effect on reducing latency when the edge server’s storage capacity is limited. In contrast, SlimCaching can skip hidden-state transmissions for tokens whose activated experts are cached locally, and can further avoid routing to the cloud whenever the required experts are collectively cached at the user and edge servers. This capability leads to substantially lower communication latency as the edge server's storage capacity increases.

\begin{figure}[!t]
	\centering
\subfigure[Latency vs. storage capacity of edge servers.]{\includegraphics[width =0.24\textwidth]{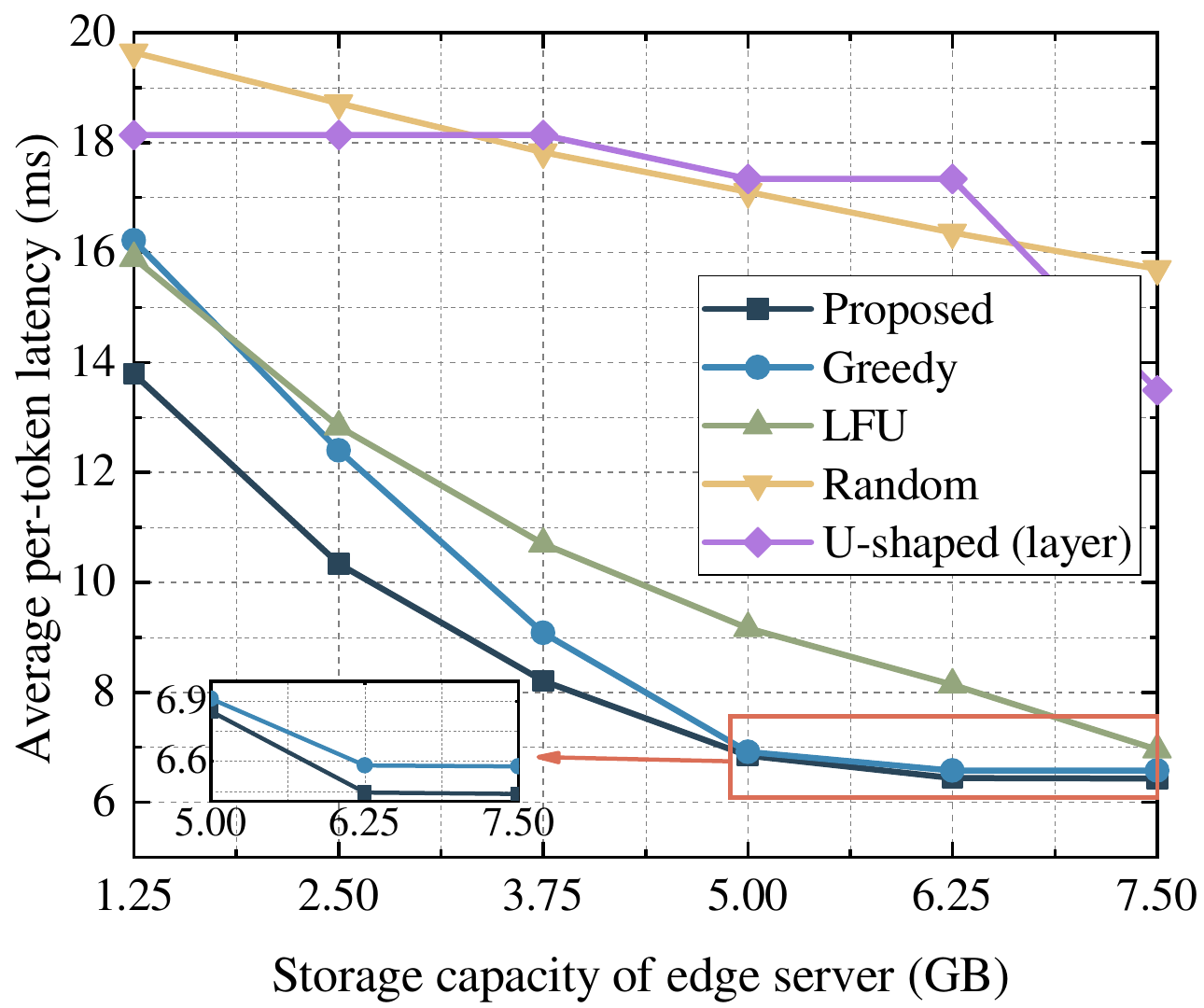}\label{fig:fig_storage_BS}}
	\subfigure[Latency vs. number of experts locally stored by each user.]{\includegraphics[width =0.24\textwidth]{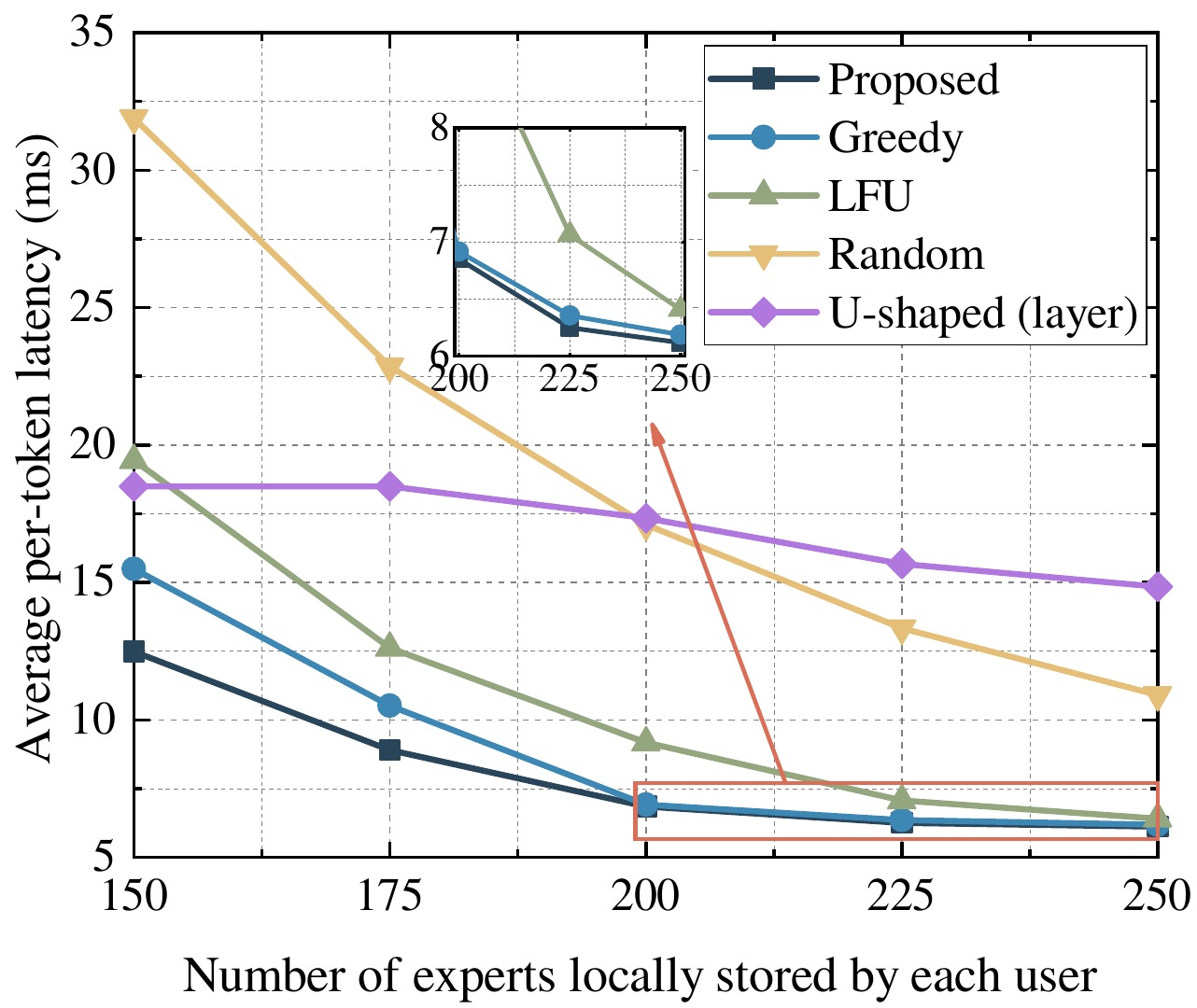}\label{fig:fig_num_local_cached_expert}}
	\caption{Latency performance comparison of algorithms under different storage capacities of edge servers and users.} 
\label{fig:combine_storage_BS_UE}  
\end{figure}

Fig. \ref{fig:fig_num_local_cached_expert} illustrates the impact of varying the number of locally cached experts at each user on the average per-token latency. As expected, increasing the user-side cache capacity reduces the latency across all algorithms, since fewer experts need to be retrieved from edge or cloud servers. Across all configurations, the proposed algorithm consistently achieves the lowest latency. The advantage of the proposed method is particularly evident when the number of locally cached experts is small, where selecting the most impactful experts at the edge servers becomes crucial for reducing inference latency. As the local caching capacity increases, the latency of all expert-based strategies gradually converges. However, the proposed method maintains a consistently lower latency across the entire range, demonstrating its ability to adapt edge-server caching decisions to the user-side caching configuration and thereby achieve more efficient expert placement.
Although the LFU and Random schemes also benefit from increased user storage, they still exhibit noticeably higher latency than the proposed method. The U-shaped SI scheme performs the worst among all algorithms, with persistently high latency and only marginal improvement as user storage grows.

Fig. \ref{fig:fig_num_request_model} shows the average per-token latency as the number of models requested by each user increases. As expected, the latency of all algorithms increases as the request load becomes heavier, since each additional model request requires retrieving a larger set of experts and intensifies competition for cached experts in the edge network.
Across all request levels, the proposed algorithm consistently achieves the lowest latency. The performance gap between the proposed method and the greedy algorithm widens as user demand increases, indicating that greedy caching becomes less effective under heavier workloads.
The LFU, Random, and U-shaped schemes show significantly higher latency across all request levels. Their performance deteriorates rapidly as the number of requested models increases.

\begin{figure}[!t]
	\centering
	\subfigure[{Latency vs. number of models requested per end user.}]{\includegraphics[width =0.235\textwidth]{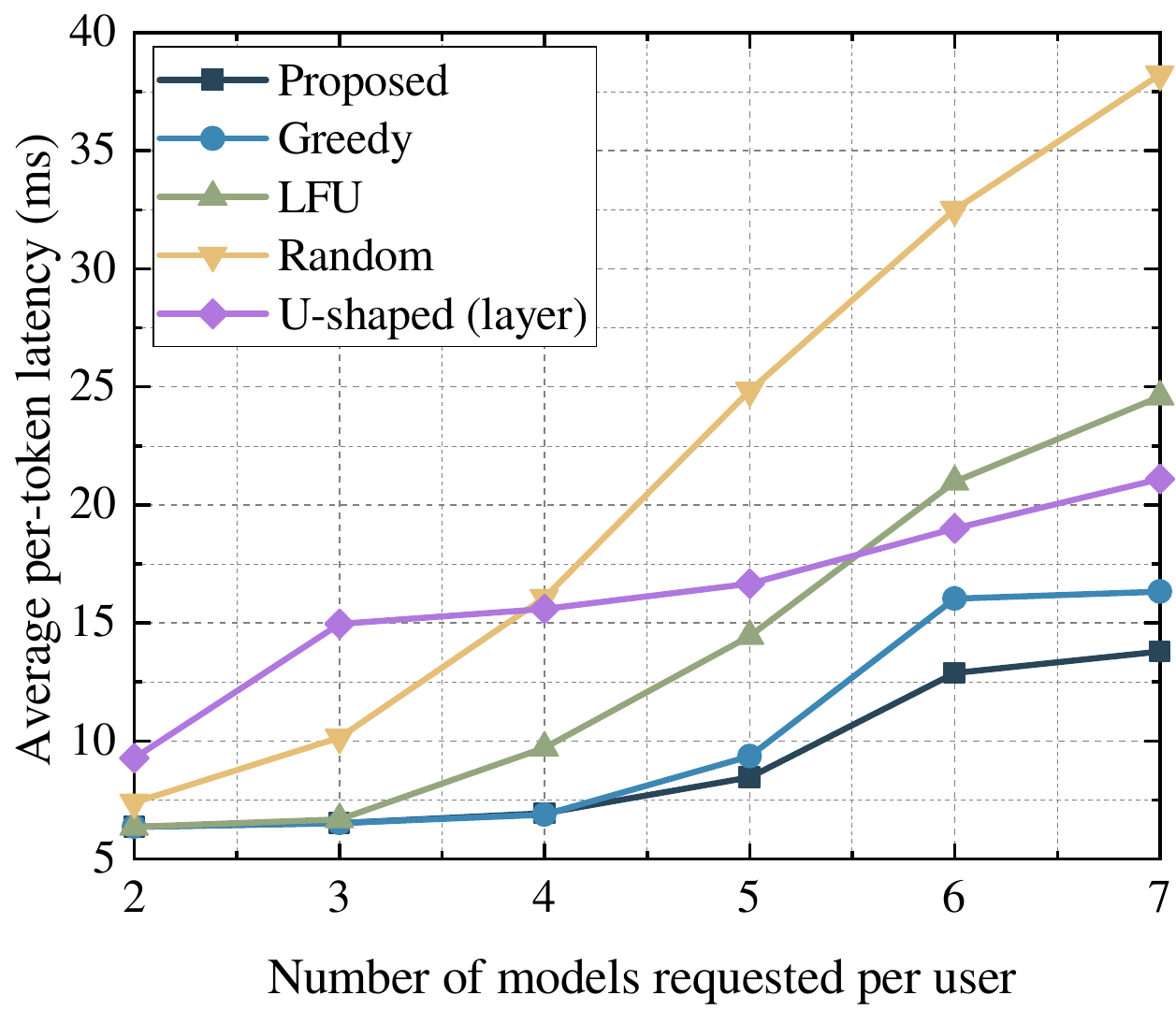}\label{fig:fig_num_request_model}}
	\subfigure[{Latency vs. bandwidth between the user and its associated edge server.}]{\includegraphics[width =0.245\textwidth]{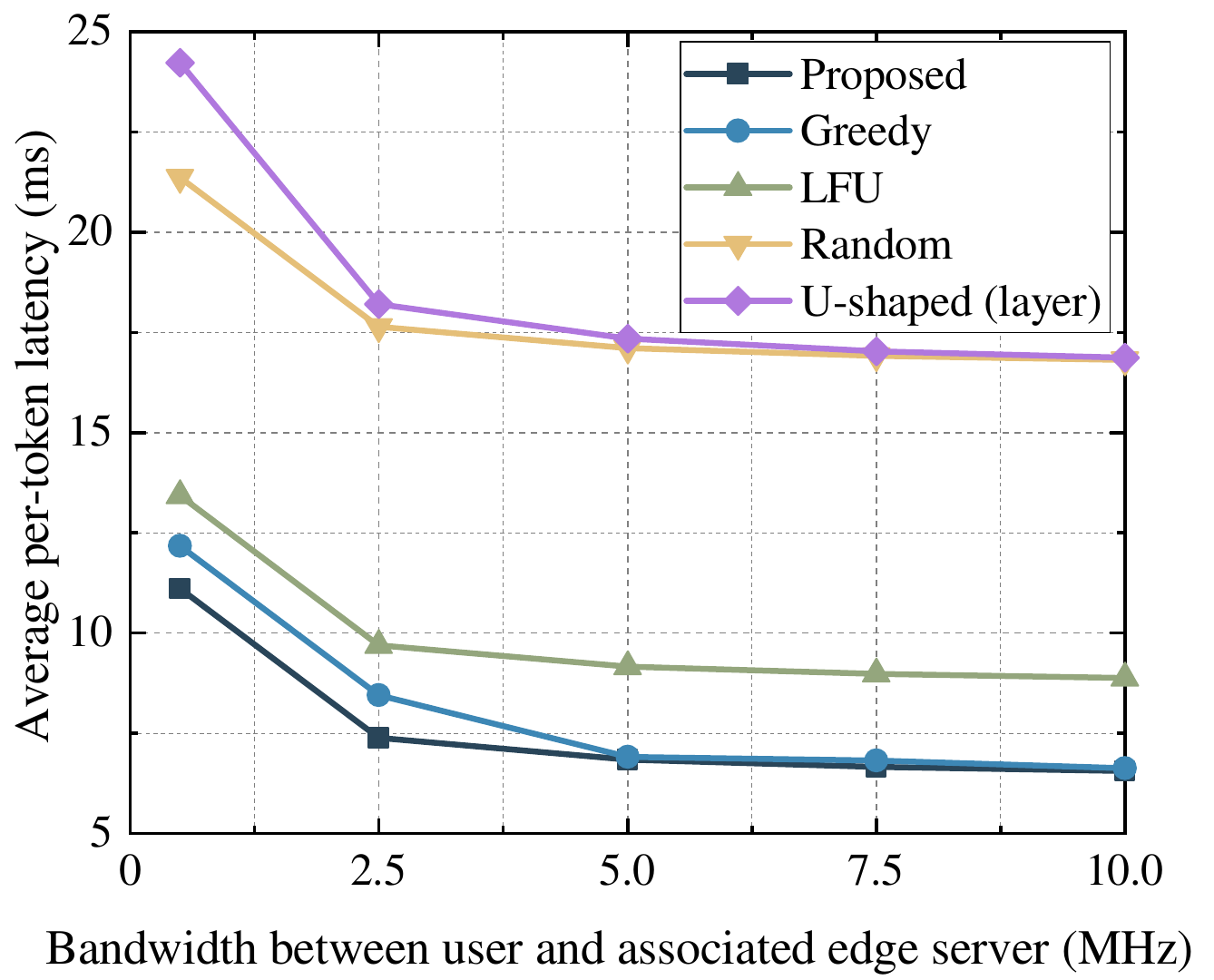}\label{fig:fig_band_U2B}}
	\caption{Latency performance comparison of algorithms under varying end-user configurations.} 
	\label{fig:latency_UEconfigure}  
\end{figure}

Fig. \ref{fig:fig_band_U2B} illustrates the average inference latency under different bandwidth values between each user and its associated edge server. As expected, the latency of all algorithms decreases when the available bandwidth increases, since higher data rates reduce the time required to upload hidden states and download expert output results.
Across all bandwidth levels, the proposed algorithm consistently achieves the lowest latency. The advantage of the proposed method is most pronounced at small bandwidths, where communication resources are scarce and inefficient expert placement can easily translate into additional transmission delay.
As the bandwidth increases, the communication bottleneck is gradually alleviated, and the curves of the expert-based schemes converge, leading to a smaller absolute gap between the proposed and greedy methods.
The LFU, Random, and U-shaped schemes exhibit significantly higher latency across all bandwidth settings. Although their latency also decreases with increasing bandwidth, they remain far above the proposed and greedy curves, indicating that their caching decisions still lead to higher overall latency even when the wireless link is less restrictive.

Fig. \ref{fig:fig_num_BS} illustrates the impact of the number of edge servers on the average inference latency. It can be observed that the latency of all algorithms decreases as the number of edge servers increases, because a larger set of distributed caches improves expert availability and reduces the inter-edge transmission distance between the serving edge server and the servers storing the required experts.
Across all configurations, the proposed algorithm consistently achieves the lowest latency. The advantage of the proposed method is most pronounced when the number of edge servers is small, where limited edge resources make coordinated expert placement especially important. In this region, greedy and LFU incur noticeably higher latency due to their less coordinated caching behavior. As more edge servers are deployed, the proposed, greedy, and LFU methods converge in performance, since all three can leverage the increased caching capacity in the edge network.
In contrast, the Random and U-shaped schemes maintain substantially higher latency across all settings, even though they also benefit from the additional edge servers.

\begin{figure}[!t]
	\centering
	\subfigure[{Latency vs. number of edge servers.}]{\includegraphics[width =0.24\textwidth]{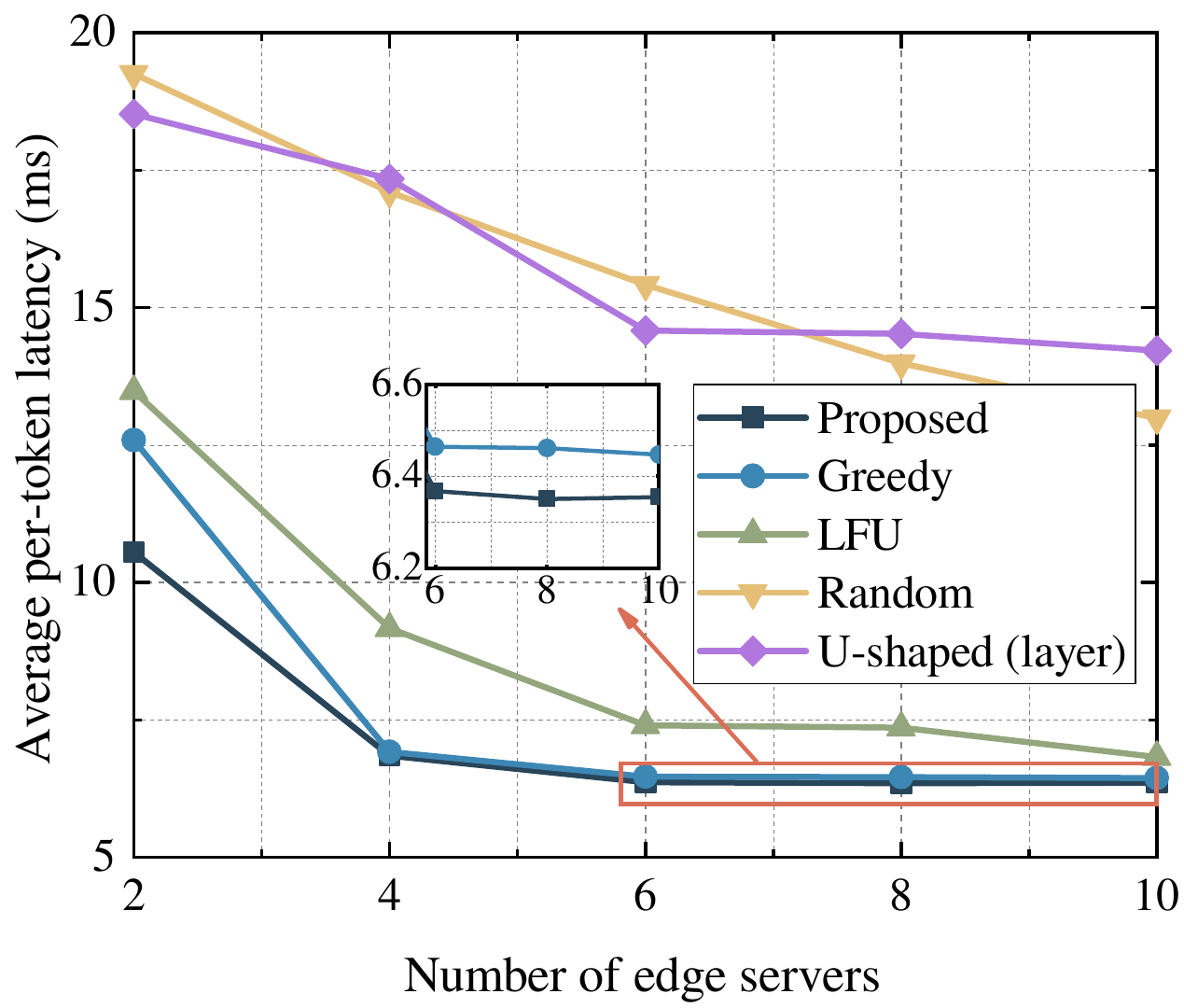}\label{fig:fig_num_BS}}
	\subfigure[{Latency vs. number of users.}]{\includegraphics[width =0.24\textwidth]{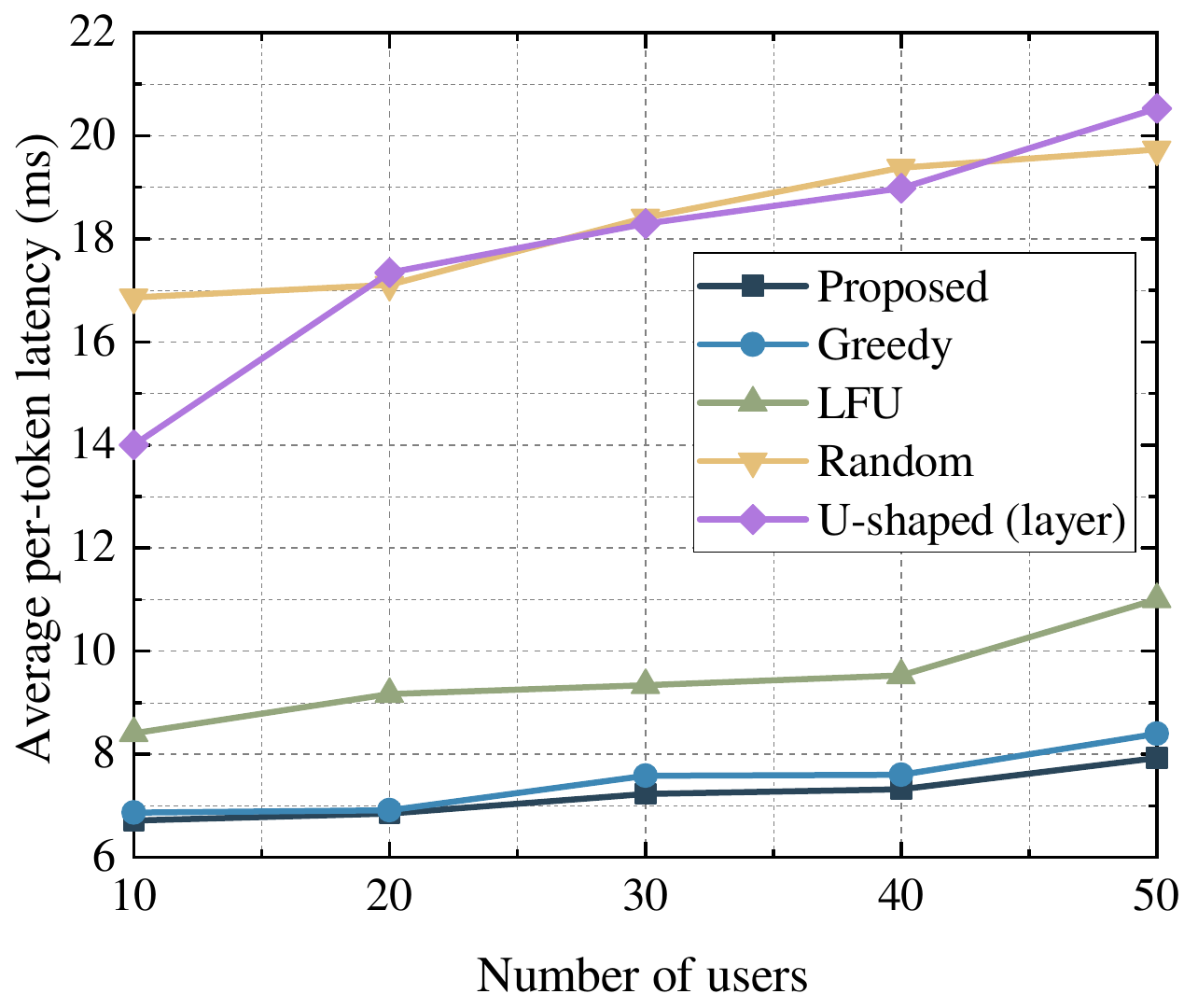}\label{fig:fig_num_UE}}
	\caption{Latency performance comparison of algorithms under different numbers of edge nodes.} 
	\label{fig}  
\end{figure}

Fig. \ref{fig:fig_num_UE} shows that the average inference latency increases for all algorithms as the number of users grows from 10 to 50. This rising trend occurs because all users share a fixed total bandwidth. When more users join the system, each user is allocated a smaller portion of the bandwidth, leading to increased communication delay. In addition, a larger user population intensifies contention for cached experts at edge servers, which further contributes to the increased overall latency.
Across all user densities, the proposed algorithm consistently achieves the lowest latency. Its performance advantage becomes more evident as the number of users increases, since coordinated expert placement becomes more important under constrained-resource scenarios. The greedy and LFU schemes also exhibit increasing latency but remain noticeably higher than the proposed method. The Random and U-shaped strategies show the highest latency across all settings, as their ineffective caching behavior becomes increasingly detrimental under congestion.

\subsection{Algorithm Running Time Comparison}

In this subsection, we discuss the impact of several key parameters on the algorithm's running time, including the storage capacity of edge server $Q$, the number of MoE models $M$ (equivalent to the number of experts $E$), and the number of edge servers $N$ and users $U$.

Fig. \ref{fig:runningtime_num_BSmodel} compares the algorithm running time of the proposed and the greedy algorithms under varying edge server's storage capacities (left) and different numbers of MoE models (right). In both cases, the proposed algorithm demonstrates significantly better computational efficiency.
Fig. \ref{fig:runtime_storage_BS} shows the algorithm running time under different edge server storage capacities. As storage increases from 1.25 GB to 7.5 GB, the running time of the greedy algorithm rises sharply, showing poor scalability. In contrast, the proposed algorithm maintains a much lower and more stable runtime over the same range.
This striking difference is as follows. Greedy explores the solution space through repeated local searches, and as storage capacity increases, the number of candidate caching combinations grows rapidly, leading to substantial computational overhead. The computational complexity of the proposed method is linear in $Q$, allowing it to scale efficiently with storage size.
Fig. \ref{fig:runtime_nummodel} illustrates the impact of the number of MoE models on the algorithm running time, exhibiting a trend similar to that observed in Fig. \ref{fig:runtime_storage_BS}. The nonlinear growth of the greedy algorithm aligns with its computational complexity of $\mathcal{O}(N^2E^2)$.

\begin{figure}[!t]
	\centering
	\subfigure[{Effect of storage capacity.}]{\includegraphics[width =0.24\textwidth]{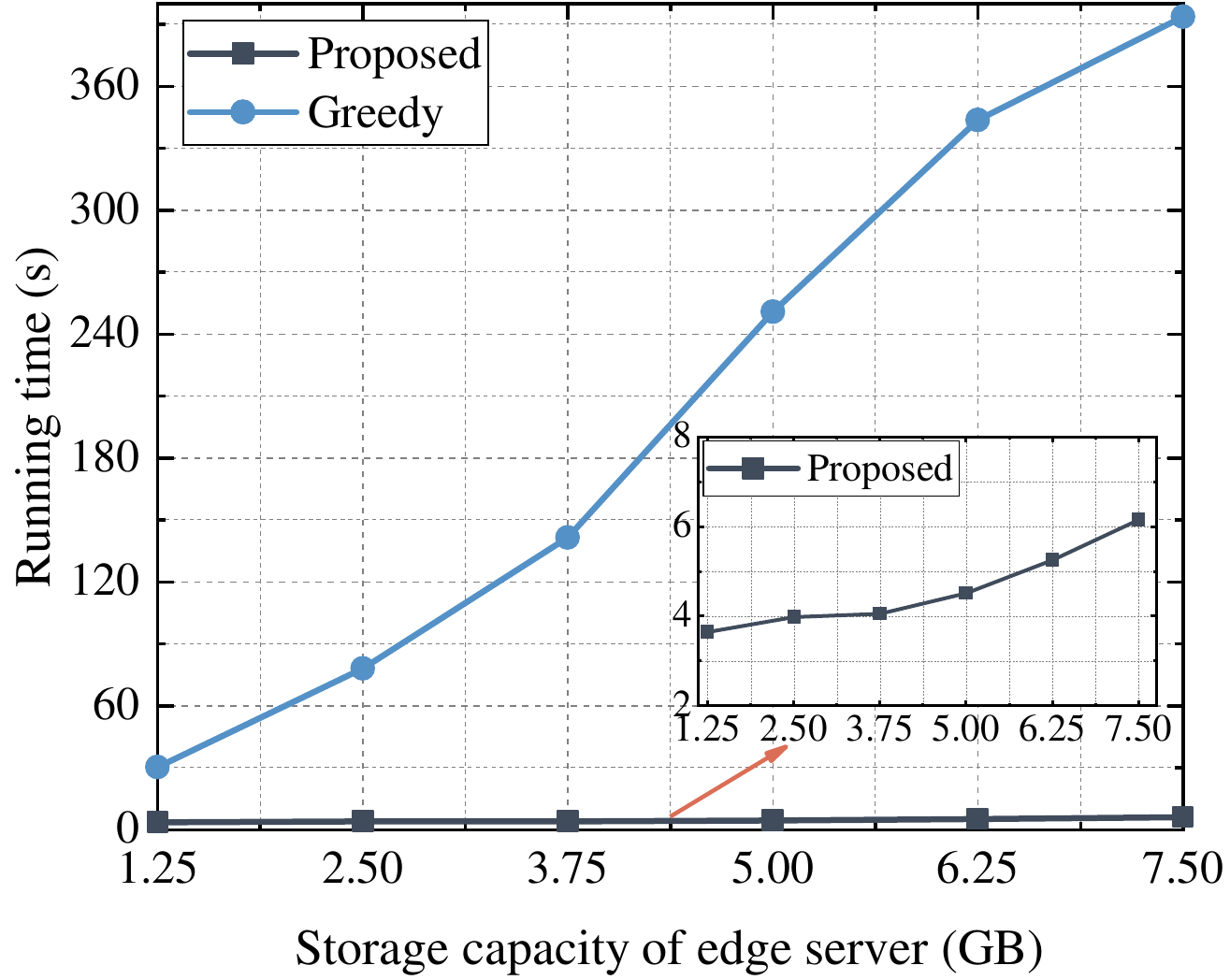}\label{fig:runtime_storage_BS}}
	\subfigure[{Effect of number of MoE models.}]{\includegraphics[width =0.24\textwidth]{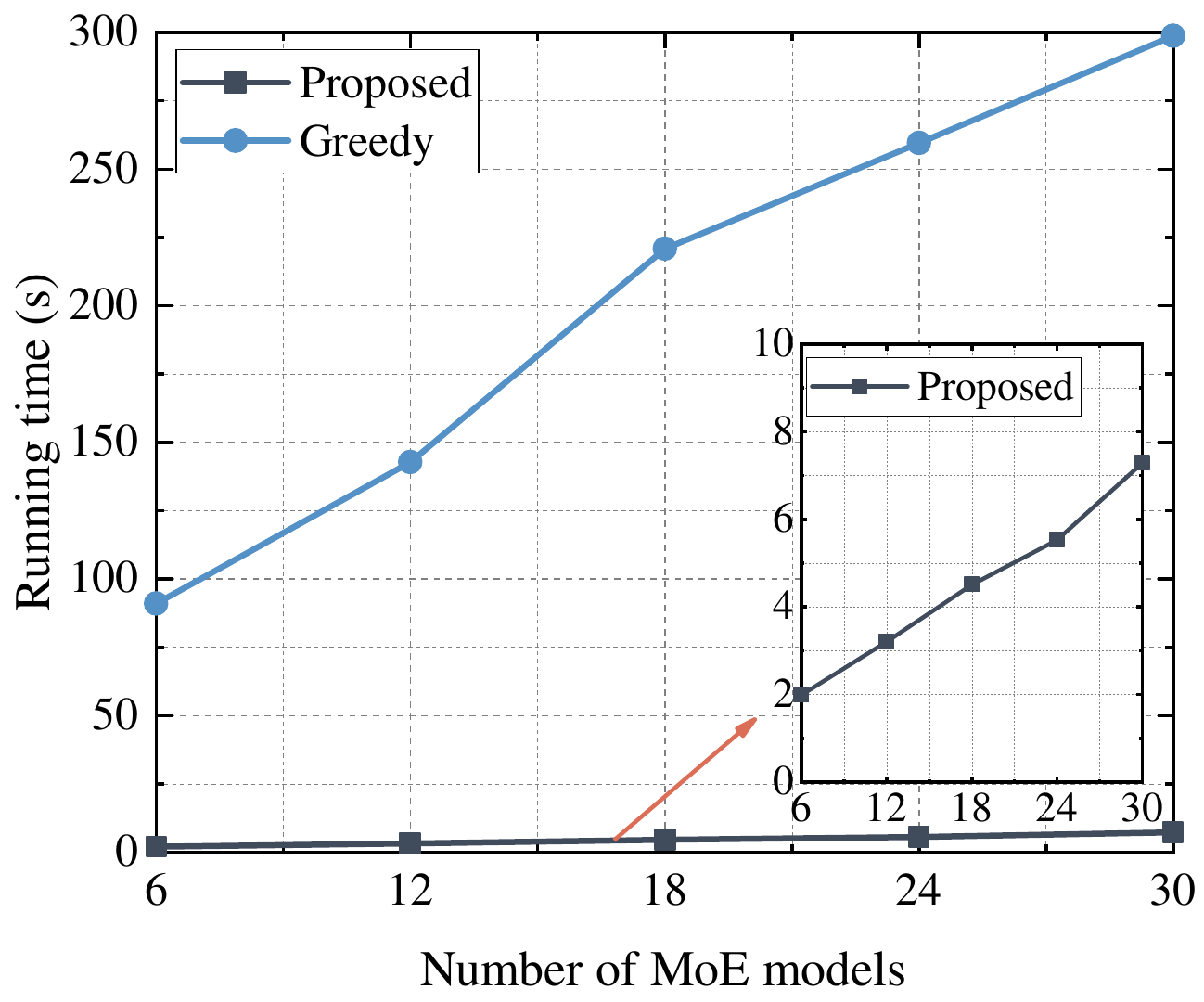}\label{fig:runtime_nummodel}}
	\caption{Algorithm running time under different storage capacities and MoE models.} 
\label{fig:runningtime_num_BSmodel}  
\end{figure}

Fig. \ref{fig:runningtime_num_BS_UE} illustrates the impact of network scale on the algorithm running time, represented by the number of edge servers (left) and users (right). Across both subfigures, the proposed algorithm consistently achieves superior scalability compared to the greedy one.
In Fig. \ref{fig:fig_runningtime_num_BS}, as the number of edge servers increases, the running time of the greedy algorithm grows rapidly, exceeding 900 s at 10 edge servers. This is due to the combinatorial increase in placement possibilities, which significantly enlarges the search space for greedy's decision-making. In contrast, the proposed method maintains a much lower execution time. 
In Fig. \ref{fig:fig_runningtime_numUE}, although users are not part of the computational complexity in both algorithms, increasing the number of users still leads to longer running time for both methods. The greedy algorithm’s running time rises sharply to over 500 s at 50 users, while the proposed algorithm remains efficient, with running time under 12 s. This is because more users introduce a larger volume of requests, which enlarges the input space and introduces overhead in utility calculation. 

\begin{figure}[!t]
	\centering
	\subfigure[{Effect of number of edge servers.}]{\includegraphics[width =0.24\textwidth]{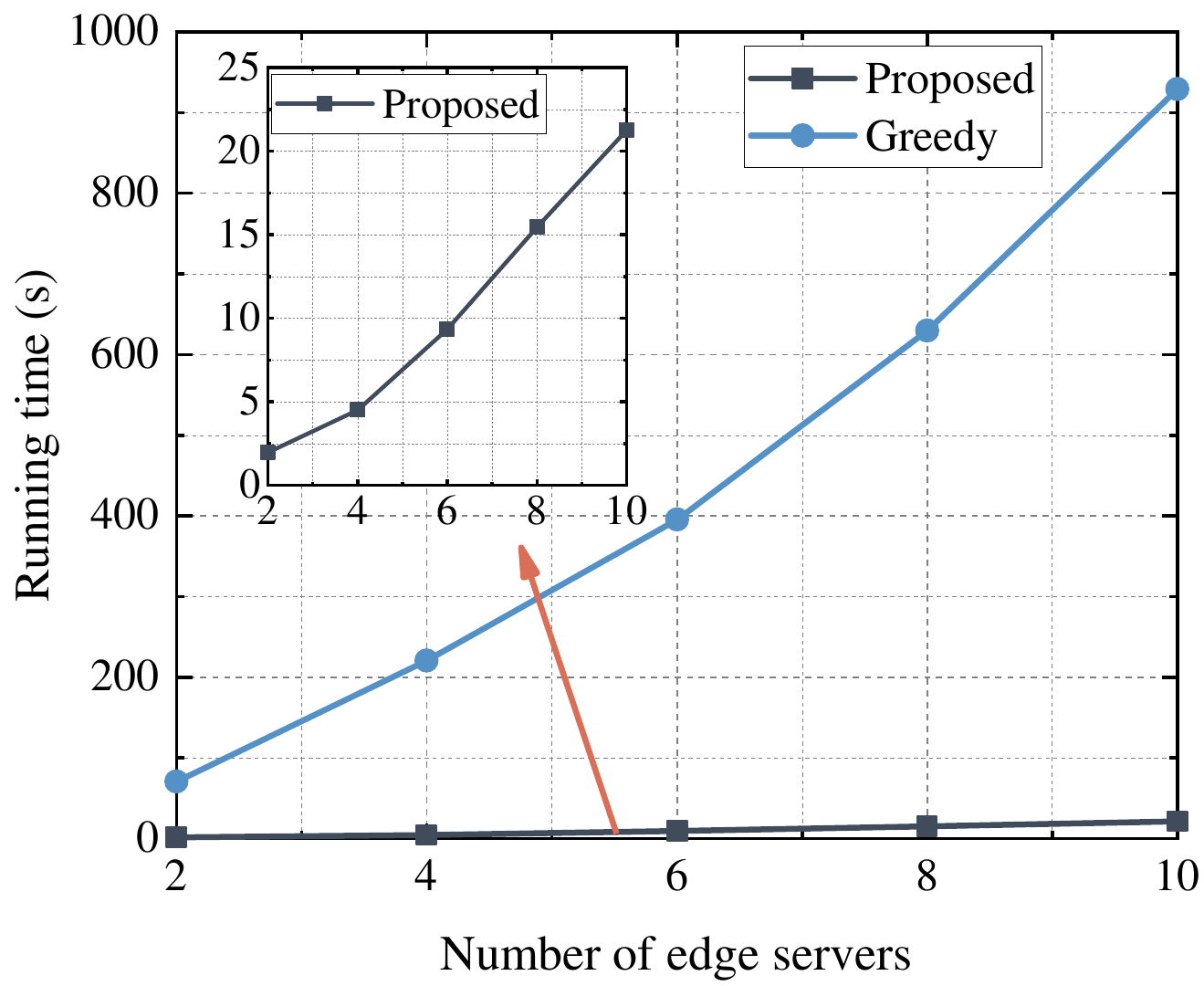}\label{fig:fig_runningtime_num_BS}}
	\subfigure[{Effect of number of users.}]{\includegraphics[width =0.24\textwidth]{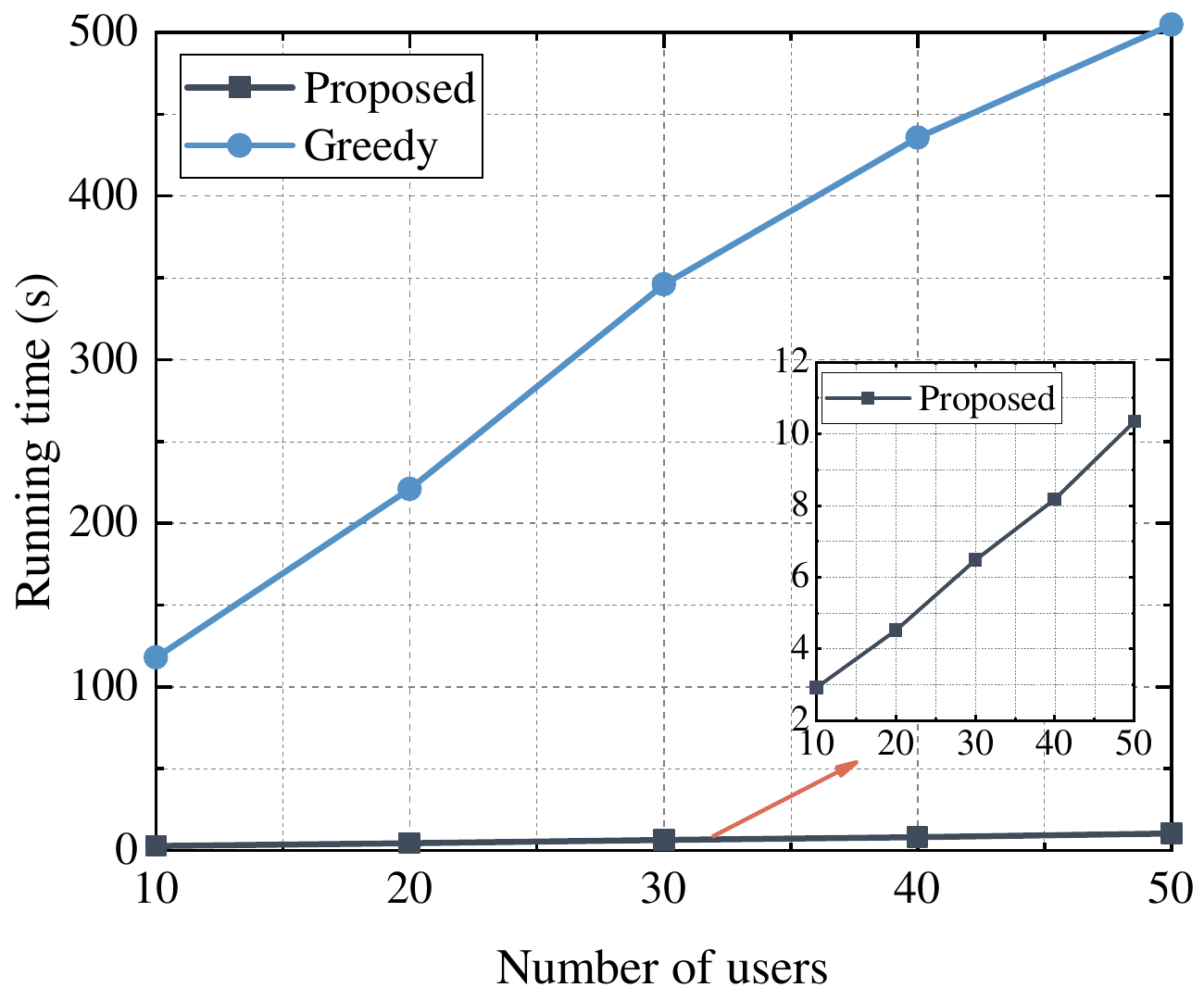}\label{fig:fig_runningtime_numUE}}
	\caption{Algorithm running time under different numbers of nodes.} 
\label{fig:runningtime_num_BS_UE}  
\end{figure}

	\section{Conclusion Remarks}\label{sec:conclusion}
In this paper, we investigate the expert caching problem to improve the efficiency of MoE inference in distributed networks. Given the caching states of users, we design a latency-minimization optimization problem to optimize the caching strategy at the edge servers. We find that the greedy-based algorithm can solve the problem when $K =1$ with a constant approximation guarantee. However, for the general case when $K \geq 1$, the non-submodularity and non-supermodularity of the problem make the greedy algorithm not applicable. To tackle this challenge, we employ a successive greedy method to decompose the original problem into multiple subproblems, which can be solved using a DP-based algorithm. Notice that the data size of expert networks in a MoE model and fine-tuned variants of a common MoE backbone is identical. We further propose an accelerated algorithm to solve the general problem with a constant approximation ratio.

The proposed expert caching scheme for efficient MoE inference introduces a novel direction for designing expert deployment strategies in storage-constrained distributed edge networks. Another promising direction is to incorporate user-level scheduling and contention-aware expert execution, so as to jointly optimize expert placement, task allocation, and GPU resource sharing under high user-load conditions. Furthermore, expert caching can be integrated with expert prefetching and token batching strategies. In this context, expert caching represents a long-term optimization problem, while expert prefetching and token batching are short-term mechanisms that can be coordinated to further enhance overall system performance.

	\ifCLASSOPTIONcaptionsoff
	\newpage
	\fi

\begin{appendices}
\section{Proof of Proposition \ref{proposition:submodular_K1}}\label{proof:prop_submodular_K1}

When $K_m=1$ for all MoE models holds, the probabilities of a user activating different experts within the same layer are independent.
Since the summation of submodular functions preserves submodularity, we only need to demonstrate that the set function $F_{u,S_{m}^{(\ell)}}\left ( \mathbf{X}  \right ) = p_{u,S_{m}^{(\ell)}} r_{u,S_{m}^{(\ell)}}\left ( \mathbf{X}  \right ) $ is submodular for a user $u$ and an expert $S_{m}^{(\ell)}$ of model $m$'s $\ell$-th layer.
Consider two caching strategy sets, $\dot{X}$ and $\hat{X}$, where $\dot{X} \subset \hat{X} \subset V$ and $V$ is the ground set which has been defined before. Under the caching strategy $\hat{X} $ and $\dot{X}$, user $u$ transmits its token to edge server $\hat{n}$ and $\dot{n}$, respectively, and $ {{T_{{n_u},\hat{n},m}} + T_{\hat{n},{n_u},m}}  \leq  {{T_{{n_u},\dot{n},m}} + T_{\dot{n},{n_u},m}}$ holds, indicating that the round-trip transmission latency between edge server $n_u$ and $\hat{n}$ is lower than that between $n_u$ and $\dot{n}$. Specifically, when $\hat{n}$ or $\dot{n}$ is identical to $n_u$, the round-trip transmission latency between edge server $n_u$ and $\hat{n}$ or between $n_u$ and $\dot{n}$ is 0.
Suppose we add an expert $S_m^{(\ell)} \in V \setminus \hat{X}$ to the edge server $n$, we discuss the following two cases:
\begin{enumerate}
    \item When $ {{T_{{n_u},n,m}} + T_{n,{n_u},m}} >{{T_{{n_u},\hat{n},m}} + T_{\hat{n},{n_u},m}}  $: 
In this case, the marginal value satisfies of caching strategy $\hat{X} $ is given by $F_{u,S_{m}^{(\ell)}}\left ( S_m^{(\ell)} \mid \hat{X} \right )= 0$.
If $ {{T_{{n_u},\dot{n},m}} + T_{\dot{n},{n_u},m}}  <  {{T_{{n_u},n,m}} + T_{n,{n_u},m}}$, the marginal value of caching strategy $\dot{X}$ remains zero. If $ {{T_{{n_u},\dot{n},m}} + T_{\dot{n},{n_u},m}}  > {{T_{{n_u},n,m}} + T_{n,{n_u},m}}$, the marginal value of $\dot{X}$ is given by $F_{u,S_{m}^{(\ell)}}\left ( S_m^{(\ell)} \mid \dot{X} \right )= 
p_{u,S_m^{\left ( \ell \right ) }}\left (  {{T_{{n_u},\dot{n},m}} + T_{\dot{n},{n_u},m}}  - {{T_{{n_u},n,m}} - T_{n,{n_u},m}} \right )  >0$.

\item When $ {{T_{{n_u},n,m}} + T_{n,{n_u},m}} < {{T_{{n_u},\hat{n},m}} + T_{\hat{n},{n_u},m}}  $: In this case, the marginal value of strategy $\hat{X}$ is given by $ F_{u,S_m^{(\ell)}}\left ( S_m^{(\ell)} \mid \hat{X} \right )= 
p_{u,S_m^{\left ( \ell \right ) }} \left ( {{T_{{n_u},\hat{n},m}} + T_{\hat{n},{n_u},m}} - {{T_{{n_u},n,m}} - T_{n,{n_u},m}}  \right )  $. 
The marginal value of strategy $\dot{X}$ is given by $ F_{u,S_m^{(\ell)}}\left ( S_m^{(\ell)} \mid \dot{X} \right )= 
p_{u,S_m^{\left ( \ell \right ) }} \left ( {{T_{{n_u},\dot{n},m}} + T_{\dot{n},{n_u},m}} - {{T_{{n_u},n,m}} - T_{n,{n_u},m}}  \right )  $. Thus, the difference in marginal gains between $\hat{X}$ and $\dot{X}$ can be expressed as
\begin{align*}
& F_{u,S_m^{\left ( \ell \right ) }}\left ( S_m^{(\ell)} \mid \hat{X} \right ) -  F_{u,S_m^{\left ( \ell \right ) }}\left ( S_m^{(\ell)} \mid \dot{X} \right )   \\
&= p_{u,S_m^{\left ( \ell \right ) }} \left ( {{T_{{n_u},\hat{n},m}} + T_{\hat{n},{n_u},m}}  -{{T_{{n_u},\dot{n},m}} - T_{\dot{n},{n_u},m}}  \right )  \leq 0,
\end{align*}
\end{enumerate}
which completes the proof.

\section{Proof of Proposition \ref{proposition:K2_nonsub_nonsuper}}\label{proof:prop_K2_nonsub_nonsuper}

We consider two caching strategy sets, $\dot{X}$ and $\hat{X}$, where $\dot{X} \subset \hat{X} \subset V$, and introduce an expert $j_m^{(\ell)}$ to the set $V \setminus \hat{X}$. 
Different from $K_m=1$, the activated experts appear in pairs when $K_m > 1$, and the introduction of any expert will affect the latency when requesting the expert pairs involving this expert.

Fig. \ref{fig:relationship_K2} shows the caching strategy after introducing a new expert $j_m^{(\ell)}$ to the sets $\dot{X}$ and $\hat{X}$. Without loss of generality, we assume that the round-trip transmission latency between edge server $A$ (i.e., $n_u$) and edge servers $B$, $C$, and $D$ are in an increasing order.
For simplicity, we assume that for model $m$'s $\ell$-th MoE layer, 
the experts except $i_{m}^{(\ell)}$ and $j_{m}^{(\ell)}$ can only be found at the cloud. Also, only the users covered by the edge server $A$ require $S_{m}^{(\ell)} = \left \{ i_{m}^{(\ell)},j_{m}^{(\ell)} \right \}$ for inference.
In this case, introducing $j_m^{(\ell)}$ will only impact the latency when users covered by $A$ require $S_{m}^{(\ell)}$. Therefore, we only need to discuss the property of $F_{u,S_{m}^{(\ell)}}\left ( \mathbf{X}  \right ) = p_{u,S_{m}^{(\ell)}} r_{u,S_{m}^{(\ell)}}\left ( \mathbf{X}  \right ) $ and then obtain the property of $F(\mathbf{X})$.

\begin{figure}[!t]
	\centering
\subfigure[$F_{u,S_{m}^{(\ell)}}\left ( j_m^{(\ell)} \mid \hat{\mathbf{X}} \right ) - F_{u,S_{m}^{(\ell)}}\left ( j_m^{(\ell)} \mid \dot{\mathbf{X}} \right ) > 0 $]{\includegraphics[width =0.47\textwidth]{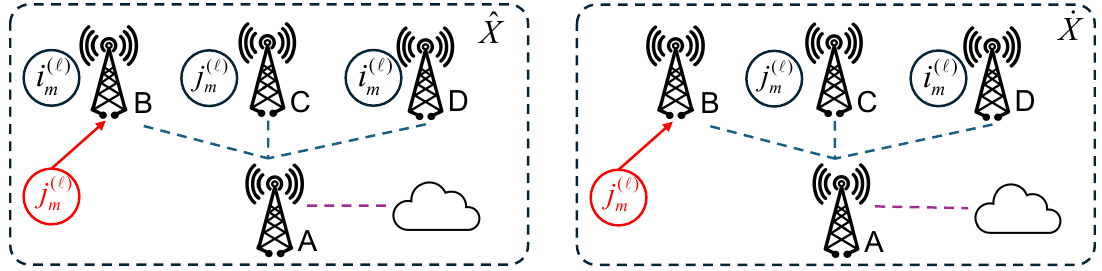}\label{fig:fig_set_K2_super}}
\subfigure[$F_{u,S_{m}^{(\ell)}}\left ( j_m^{(\ell)} \mid \hat{\mathbf{X}} \right ) - F_{u,S_{m}^{(\ell)}}\left ( j_m^{(\ell)} \mid \dot{\mathbf{X}} \right ) < 0 $]{\includegraphics[width =0.47\textwidth]{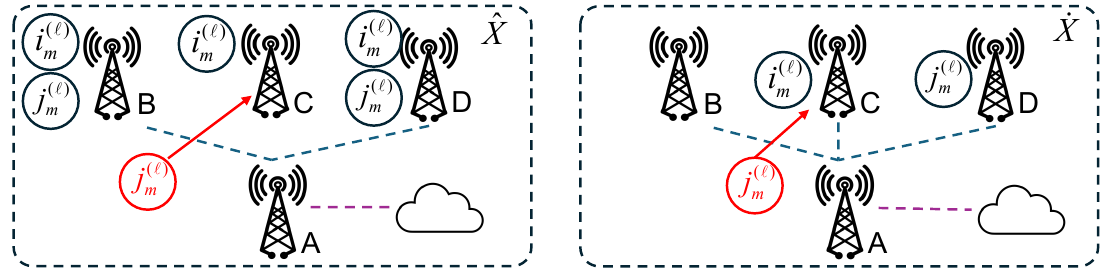}\label{fig:fig_set_K2_sub}}
	\caption{Different relationships between the marginal values of $\hat{\mathbf{X}}$ and $\dot{\mathbf{X}}$.} 
\label{fig:relationship_K2}  
\end{figure}

First, we discuss the case in Fig. \ref{fig:fig_set_K2_super}.
In order to minimize the E2E latency, for the set $\hat{X}$, the token will be transmitted to edge server $\{B, C\}$ and $B$, before and after adding $j_m^{(\ell)}$ to edge server $B$, respectively.
For the set $\dot{X}$, the token will be transmitted to edge server $\{C, D\}$ and $\{B, D\}$, before and after adding $j_m^{(\ell)}$ to edge server $B$, respectively.
The marginal values of $\hat{X}$ and $\dot{X}$ are given by $F_{u,S_m^{(\ell)}}\left ( j_m^{(\ell)} \mid \hat{X} \right ) =  p_{u,S_m^{(\ell)}}\left( T_{A,C,m} +  T_{C,A,m} - T_{B,A,m} + T_{{\rm{E}},m}^{\rm{CP}}\right) $
and $F_{u,S_m^{(\ell)}}\left ( j_m^{(\ell)} \mid \dot{X} \right ) = p_{u,S_m^{(\ell)}} ( T_{A,C,m} + T_{C,A,m} -T_{A,B,m}  -T_{B,A,m} )$, respectively.
In this case, we have $F_{u,S_m^{(\ell)}}\left ( j_m^{(\ell)} \mid \hat{X} \right ) - F_{u,S_m^{(\ell)}}\left ( j_m^{(\ell)} \mid \dot{X} \right ) > 0 $, and then $F\left ( j_m^{(\ell)} \mid \hat{X} \right ) - F\left ( j_m^{(\ell)} \mid \dot{X} \right ) > 0$ holds.


Then, we discuss the scenario in Fig. \ref{fig:fig_set_K2_sub}.
For the set $\hat{X}$, before and after adding $j_m^{(\ell)}$ to edge server $C$, the token will be processed at edge server $B$ to minimize the E2E latency. Thus, the marginal value of $\hat{X}$ is given by
$F_{u,S_m^{(\ell)}}\left ( j_m^{(\ell)} \mid \hat{X} \right ) = 0$.
For the set $\dot{X}$, the token will be transmitted to edge server $\{ C,D\}$ and $C$, before and after adding $j_m^{(\ell)}$ to edge server $C$, respectively.
The marginal values of $\dot{X}$ is given by $F_{u,S_m^{(\ell)}}\left ( j_m^{(\ell)} \mid \dot{X} \right )  = p_{u,S_m^{(\ell)}}(T_{A,D,m} + T_{D,A,m} -T_{C,A,m}+T_{{\rm{E}},m}^{\rm{CP}})$. Thus, we have $F_{u,S_m^{(\ell)}}\left ( j_m^{(\ell)} \mid \hat{X} \right ) - F_{u,S_m^{(\ell)}}\left ( j_m^{(\ell)} \mid \dot{X} \right ) <0$, and then $F\left ( j_m^{(\ell)} \mid \hat{X} \right ) - F\left ( j_m^{(\ell)} \mid \dot{X} \right ) < 0$ holds. At this point, when $K_m =2$, we verify that $F\left ( j_m^{(\ell)} \mid \hat{X} \right ) - F\left ( j_m^{(\ell)} \mid \dot{X} \right ) $ is either positive or negative, leading to the property of non-submodularity and non-supermodularity.

    Next, we extend this property to a more general case for $K_m>1$. The sketch proof is as follows: the expert $i_m^{(\ell)}$ in Fig. \ref{fig:relationship_K2}  is replaced by all the required experts excluding $j_m^{(\ell)}$, i.e., the expert group including experts $S_m^{(\ell)} \setminus j_m^{(\ell)}$. With a similar method to the case $K_m=2$, we can prove that $F\left ( j_m^{(\ell)} \mid \hat{X} \right ) - F\left ( j_m^{(\ell)} \mid \dot{X} \right ) $ may also be either positive or negative. 

    Therefore, when $K_m >1$, the non-submodularity and non-supermodularity always hold for $\mathcal{P}1$.

\section{Proof of Proposition \ref{prop:subproblem_ratio}}\label{proof:prop_subproblem_ratio}

 Since the optimal solution of $\mathcal{P}3_n$ can be achieved with DP algorithm,  ${\bar F}_n\left ( \ddot { \mathbf{X} }_n \right ) \geq {\bar F}_n\left ( \tilde{\mathbf{X}}_n^{\ast} \right )$ holds. According to Lemma \ref{lemma:property_super}, we have 
\begin{equation}
    \begin{split}
        & {\tilde F}_n^{\mathrm{mod} }\left(\tilde{\mathbf{X}}_n\right) + {\tilde F}_n^{\mathrm{super} }\left(\tilde{\mathbf{X}}_n \mid \emptyset \right)  \geq {\bar F}_n  \left ( \tilde{\mathbf{X}}_n \right ) \\
        & \geq {\tilde F}_n^{\mathrm{mod} }\left(\tilde{\mathbf{X}}_n\right) + \left ( 1-\kappa_g \right ) {\tilde F}_n^{\mathrm{super} }\left(\tilde{\mathbf{X}}_n \mid \emptyset \right).
    \end{split}
\end{equation}
Then, the following inequalities hold, 
\begin{align*}
&  {\tilde F}_n\left ( \ddot { \mathbf{X} }_n \right )  \geq {\bar F}_n\left ( \ddot { \mathbf{X} }_n \right )  +{\tilde F}_n^{\mathrm{super} }\left ( \emptyset  \right ) \\
& \geq {\bar F}_n\left ( \tilde{\mathbf{X}}_n^{\ast} \right )  +{\tilde F}_n^{\mathrm{super} }\left ( \emptyset  \right ) \\
& \geq {\tilde F}_n^{\mathrm{mod} }\left(\tilde{\mathbf{X}}_n^{\ast}\right) + \left ( 1-(\kappa_g)_n \right ) {\tilde F}_n^{\mathrm{super} }\left(\tilde{\mathbf{X}}_n^{\ast} \mid \emptyset \right) +{\tilde F}_n^{\mathrm{super} }\left ( \emptyset  \right ) \\
& = {\tilde F}_n^{\mathrm{mod} }\left ( \tilde{\mathbf{X}}_n^{\ast} \right ) + \left ( 1-(\kappa_g)_n \right ) \left ( {\tilde F}_n^{\mathrm{super} }\left ( \tilde{\mathbf{X}}_n^{\ast}  \right ) - {\tilde F}_n^{\mathrm{super} }\left ( \emptyset  \right ) \right ) \\
& + {\tilde F}_n^{\mathrm{super} }\left ( \emptyset  \right ) \\
& \geq  {\tilde F}_n^{\mathrm{mod} }\left ( \tilde{\mathbf{X}}_n^{\ast} \right ) + \left ( 1-(\kappa_g)_n \right ) {\tilde F}_n^{\mathrm{super} }\left ( \tilde{\mathbf{X}}_n^{\ast}  \right ) \\
& \geq \left ( 1-(\kappa_g)_n \right ){\tilde F}_n\left ( \tilde{\mathbf{X}}_n^{\ast} \right ).
\end{align*}

\section{Proof of Theorem \ref{theorem:approximation_ratio_general_case}}\label{proof:theorem_approximation_ratio_general_case}

Let $ \ddot{\mathbf{X} }_n $ and $ \tilde{\mathbf{X}} ^{\ast }_n$  denote the solution obtained by the proposed expert caching algorithm for $K \geq 1$ and the globally optimal solution when solving $\mathcal{P}2_n$. 
	We introduce a comparison set $\dot {\mathbf{X} }_n$, which satisfies
\begin{equation}
	\dot{\mathbf{X} }_n = \left\{ \begin{array}{l}
		\tilde{\mathbf{X}} ^{\ast }_n , \quad {\rm{if}} \; \ddot{\mathbf{X} }_n   \ne  \tilde{\mathbf{X}} ^{\ast }_n, \\
		\emptyset , \quad {\rm{otherwise}}.
	\end{array} \right.
\end{equation}
With (\ref{equ:lower_bound_subproblem}), we have $\left ( 1-\kappa_g^{\max} \right ) {\tilde F}_n\left ( \dot{\mathbf X}_n \right ) \leq \left ( 1-(\kappa_g)_n \right ) {\tilde F}_n\left ( \dot{\mathbf X}_n \right ) \leq {\tilde F}_n\left ( \ddot { \mathbf{X} }_n \right )$ for $\forall n \in \mathcal{N}$. Also, 
$\left ( 1-\kappa_g^{\max} \right ) {\tilde F}_n\left ( \tilde{\mathbf X}_n^{\ast} \right ) \leq {\tilde F}_n\left ( \ddot { \mathbf{X} }_n \right ) +\left ( 1-\kappa_g^{\max} \right ) {\tilde F}_n\left ( \dot{\mathbf X}_n \right )$ holds. Therefore, the following inequality holds,
\begin{equation}
    \begin{split}
        & \sum_{n \in \mathcal{N}}\left ( 1-\kappa_g^{\max} \right ) \tilde{F}_n({\tilde{\mathbf X}}_n^{\ast}) \\
        & \leq \sum_{n \in \mathcal{N} } {\tilde F}_n\left ( \ddot { \mathbf{X} }_n \right ) + \left ( 1-\kappa_g^{\max} \right ) {\tilde F}_n\left ( \dot{\mathbf X}_n \right ) \leq 2 \sum_{n \in \mathcal{N} }{\tilde F}_n\left ( \ddot { \mathbf{X} }_n \right ).
    \end{split}
\end{equation}
Then, we have $\sum_{n \in \mathcal{N} }{\tilde F}_n\left ( \ddot { \mathbf{X} }_n \right ) \geq \frac{ 1-\kappa_g^{\max} }{2}\sum_{n \in \mathcal{N}} \tilde{F}_n({\tilde{\mathbf X}}_n^{\ast}) $. Since (\ref{equ:union_solution}) holds, we have $F\left ( \ddot{\mathbf{X} }  \right ) \geq 
\frac{1-\kappa_g^{\max}}{2}  F\left ( \tilde{\mathbf{X}} ^{\ast }  \right )$.

In the special case $N=1$, the approximation ratio can be obtained directly according to Proposition \ref{prop:subproblem_ratio}.

\section{Proof of Proposition \ref{prop:approximationratio_condition}}\label{proof:prop_approximationratio_condition}

Based on the assumptions stated in Proposition~\ref{prop:approximationratio_condition}, we analyze the single-edge-server and multi-edge-server scenarios separately.
	
(1) \textit{Single-edge-server scenario: }
Consider the case $N=1$. 
The marginal gain of placing expert $i_m^{(\ell)}$ at the edge server with respect to the empty set is ${\tilde F}_n^{\mathrm{super} }(i_m^{(\ell)} \mid \emptyset) =\sum\limits_{u \in \mathcal{U}} \sum\limits_{m \in \mathcal{M}} p_{u,m}p_{u,i_{m}^{(\ell)}} \left ( T_{{\rm{C}},n,m}-T_{\mathrm{E},m}^{\mathrm{CP} } \right ) $. The marginal gain evaluated at the complement set $V \setminus  i_m^{(\ell)}$ is ${\tilde F}_n^{\mathrm{super} }(i_m^{(\ell)} \mid V \setminus  i_m^{(\ell)} ) =\sum\limits_{u \in \mathcal{U}} \sum\limits_{m \in \mathcal{M}} p_{u,m}p_{u,i_{m}^{(\ell)}} \left ( T_{{\rm{C}},n,m}+T_{n,{\rm{C}},m}+T_{\mathrm{C},m}^{\mathrm{CP} } \right ) $. Since both $T_{\mathrm{E},m}^{\mathrm{CP} }$ and $T_{\mathrm{C},m}^{\mathrm{CP} }$ are much lower than $T_{{\rm{C}},n,m}$, their ratio satisfies $\frac{{\tilde F}_n^{\mathrm{super} }(i_m^{(\ell)} \mid \emptyset)}{{\tilde F}_n^{\mathrm{super} }(i_m^{(\ell)} \mid V \setminus  i_m^{(\ell)} )}\approx \frac{1}{2}$ and $\left ( \kappa _g \right )_1 \approx \frac{1}{2}$. Thus, for $N=1$, the curvature satisfies $\left ( \kappa _g \right )_1 \approx \frac{1}{2}$, and therefore the proposed algorithm guarantees a $ \frac{1}{2}$-approximation.

(2) \textit{Multi-edge-server scenario:} 
We now extend the analysis of the approximation ratio to the general case with multiple edge servers. Let $\mathcal{U}_n$ denote the set of users associated with edge server $n$. We analyze the curvature $\left( \kappa_g \right)_n$ when solving the subproblem $\mathcal{P}2_n$ at edge server $n$.

 First, we consider the subproblem $\mathcal{P}2_1$ for the first edge server.
The marginal gain when adding expert $i_m^{(\ell)}$ to the empty set is decomposed into
\[ {\tilde F}_n^{\mathrm{super} }(i_m^{(\ell)} \mid \emptyset) = {\tilde F}_{\mathcal{U}_n }^{\mathrm{super} }(i_m^{(\ell)} \mid \emptyset)+\sum_{n' \in \mathcal{N}\setminus n } {\tilde F}_{\mathcal{U}_{n'} }^{\mathrm{super} }(i_m^{(\ell)} \mid \emptyset),\]
where ${\tilde F}_{\mathcal{U}_n }^{\mathrm{super} }(i_m^{(\ell)} \mid \emptyset) = \sum\limits_{u \in \mathcal{U}_n} \sum\limits_{m \in \mathcal{M}}
p_{u,m} p_{u,i_{m}^{(\ell)}} \left ( T_{{\rm{C}},n,m}-T_{\mathrm{E},m}^{\mathrm{CP} } \right )$ represents the marginal gain contributed by the users in the set $\mathcal{U}_n$ and $ {\tilde F}_{\mathcal{U}_{n'} }^{\mathrm{super} }(i_m^{(\ell)} \mid \emptyset) = \sum\limits_{u \in \mathcal{U}_n'} \sum\limits_{m \in \mathcal{M}} p_{u,m}p_{u,i_{m}^{(\ell)}}\left ( T_{{\rm{C}},n,m}-T_{n',n,m}-T_{n,n',m} - T_{\mathrm{E},m}^{\mathrm{CP} }\right )$ denotes the marginal gain contributed by the users covered by other edge server $n' \neq n$.

The marginal gain evaluated at the complement set $V \setminus  i_m^{(\ell)}$ is decomposed into
\begin{equation*}
	\begin{split}
		&    {\tilde F}_n^{\mathrm{super} }(i_m^{(\ell)} \mid V \setminus  i_m^{(\ell)} ) \\
		& = {\tilde F}_{\mathcal{U}_n }^{\mathrm{super} }(i_m^{(\ell)} \mid V \setminus  i_m^{(\ell)} )+\sum_{n' \in \mathcal{N}\setminus n } {\tilde F}_{\mathcal{U}_{n'} }^{\mathrm{super} }(i_m^{(\ell)} \mid V \setminus  i_m^{(\ell)} ),
	\end{split}
\end{equation*}
where $ {\tilde F}_{\mathcal{U}_n }^{\mathrm{super} }(i_m^{(\ell)} \mid V \setminus  i_m^{(\ell)} ) = \sum\limits_{u \in \mathcal{U}_n} \sum\limits_{m \in \mathcal{M}} p_{u,m}p_{u,i_{m}^{(\ell)}} \left ( T_{{\rm{C}},n,m}+T_{n,{\rm{C}},m}+T_{\mathrm{C},m}^{\mathrm{CP} } \right ) $ and ${\tilde F}_{\mathcal{U}_{n'} }^{\mathrm{super} }(i_m^{(\ell)} \mid V \setminus  i_m^{(\ell)} ) =  \sum\limits_{u \in \mathcal{U}_{n'}} \sum\limits_{m \in \mathcal{M}} p_{u,m}p_{u,i_{m}^{(\ell)}} \left( T_{{\rm{C}},n,m}+T_{n,{\rm{C}},m} - T_{n,n',m} +T_{\mathrm{C},m}^{\mathrm{CP} } \right )$. 
When the cloud-edge latency is much larger than the inter-edge latency, we obtain $\frac{{\tilde F}_n^{\mathrm{super} }(i_m^{(\ell)} \mid \emptyset)}{{\tilde F}_n^{\mathrm{super} }(i_m^{(\ell)} \mid V \setminus  i_m^{(\ell)} )}\approx \frac{1}{2}$.

Next, conditioned on the caching decisions already made for edge servers $1, \ldots, n-1$, we analyze the subproblem $\mathcal{P}2_n$ for edge server $n$.
The marginal gain is decomposed into contributions from users within $\mathcal{U}_n$ and those outside $\mathcal{U}_n$. We focus on the case where the expert $i_m^{(\ell)}$ considered for placement at edge server $n$ has already been cached at some earlier edge server; otherwise, the marginal-gain analysis reduces to that for the first edge server.

For the users in $\mathcal{U}_n$, the marginal gains can be expressed as $ {\tilde F}_{\mathcal{U}_n}^{\mathrm{super} }(i_m^{(\ell)} \mid \emptyset)=    \sum\limits_{u \in \mathcal{U}_n} \sum\limits_{m \in \mathcal{M}} p_{u,m}p_{u,i_{m}^{(\ell)}} \left ( T_{n',n,m}-T_{\mathrm{E},m}^{\mathrm{CP} } \right )$ and ${\tilde F}_{\mathcal{U}_n}^{\mathrm{super} }(i_m^{(\ell)} \mid V \setminus  i_m^{(\ell)} )  = \sum\limits_{u \in \mathcal{U}_n} \sum\limits_{m \in \mathcal{M}} p_{u,m}p_{u,i_{m}^{(\ell)}} \left ( T_{n,n',m}+T_{n',n,m}+T_{\mathrm{E},m}^{\mathrm{CP} } \right )$, respectively, where $n'$ is the edge server that already stores expert $i_m^{(\ell)}$  before placement at edge server $n$ and achieves the shortest round-trip latency to $n$. Then, we have $\frac{{\tilde F}_{\mathcal{U}_n}^{\mathrm{super} }(i_m^{(\ell)} \mid \emptyset)}{{\tilde F}_{\mathcal{U}_n}^{\mathrm{super} }(i_m^{(\ell)} \mid V \setminus  i_m^{(\ell)} )} \approx \frac{1}{2}  $.

For the users covered by edge server $n'$ (i.e., $u \in \mathcal{U}_{n'}$), the marginal gain when adding expert $i_m^{(\ell)}$ to the empty set is zero unless placing $i_m^{(\ell)}$ at edge server $n$ causes requests originally served by other edge servers to be redirected to edge server $n$ and such redirection strictly reduces latency. Let $\tilde{n}$ be the edge server that possesses expert $i_{m}^{(\ell)}$ before it is added to edge server $n$ and achieves the shortest round-trip E2E latency to $n'$ among all such servers. If redirection occurs for users in $\mathcal{U}_{n'}$, we have $ {\tilde F}_{\mathcal{U}_{n' }}^{\mathrm{super} }(i_m^{(\ell)} \mid \emptyset) = \sum\limits_{u \in \mathcal{U}_{n'}} \sum\limits_{m \in \mathcal{M}} p_{u,m}p_{u,i_{m}^{(\ell)}}\left( T_{\tilde{n},n',m}-T_{n',n,m} - T_{n,n',m} - T_{\mathrm{E},m}^{\mathrm{CP} } \right)$ and ${\tilde F}_{\mathcal{U}_{n'} }^{\mathrm{super} }(i_m^{(\ell)} \mid V \setminus  i_m^{(\ell)} ) = 	\sum\limits_{u \in \mathcal{U}_{n'}} \sum\limits_{m \in \mathcal{M}} p_{u,m}p_{u,i_{m}^{(\ell)}}\left ( T_{n',\tilde{n}, m}+T_{\tilde{n}, n',m}-T_{n,n',m}+ T_{\mathrm{E},m}^{\mathrm{CP} } \right )$.

Because rediction occurs only when it reduces latency, the resulting ratio is strictly smaller than $\frac{1}{2}$, i.e., $\frac{  {\tilde F}_{\mathcal{U}_{n' }}^{\mathrm{super} }(i_m^{(\ell)} \mid \emptyset) }{{\tilde F}_{\mathcal{U}_{n'} }^{\mathrm{super} }(i_m^{(\ell)} \mid V \setminus  i_m^{(\ell)} )}  < \frac{1}{2}  $.

Combining $\frac{{\tilde F}_{\mathcal{U}_n}^{\mathrm{super} }(i_m^{(\ell)} \mid \emptyset)}{{\tilde F}_{\mathcal{U}_n}^{\mathrm{super} }(i_m^{(\ell)} \mid V \setminus  i_m^{(\ell)} )} \approx \frac{1}{2}  $ and $\frac{  {\tilde F}_{\mathcal{U}_{n' }}^{\mathrm{super} }(i_m^{(\ell)} \mid \emptyset) }{{\tilde F}_{\mathcal{U}_{n'} }^{\mathrm{super} }(i_m^{(\ell)} \mid V \setminus  i_m^{(\ell)} )}  < \frac{1}{2}  $, the overall ratio for edge server $n$ satisfies $\frac{{\tilde F}_n^{\mathrm{super} }(i_m^{(\ell)} \mid \emptyset)}{{\tilde F}_n^{\mathrm{super} }(i_m^{(\ell)} \mid V \setminus  i_m^{(\ell)} )}  \in \left ( 0, \frac{1}{2} \right ) $, which implies $\kappa_{g}^{\max} \in \left ( \frac{1}{2} , 1 \right ) $. Therefore, the relationship reflecting the approximation ratio in Theorem \ref{theorem:approximation_ratio_general_case} satisfies $F\left ( \ddot{\mathbf{X} }  \right ) \geq 
\frac{1-\kappa_g^{\max}}{2}  F\left ( \tilde{\mathbf{X}} ^{\ast }  \right ) > \frac{1}{4}F\left ( \tilde{\mathbf{X}} ^{\ast }  \right )$. That is, the proposed algorithm can guarantee a $\frac{1}{4}$-approximation ratio in the multi-server scenario.

\end{appendices}

	\bibliographystyle{IEEEtran}
	\bibliography{IEEEabrv,reference}
	
\end{document}